\documentclass{article}

\PassOptionsToPackage{numbers,sort&compress}{natbib}
\usepackage[preprint]{neurips_2026}

\usepackage[utf8]{inputenc} % allow utf-8 input
\usepackage[T1]{fontenc}    % use 8-bit T1 fonts
\usepackage{hyperref}       % hyperlinks
\usepackage{url}            % simple URL typesetting
\usepackage{booktabs}       % professional-quality tables
\usepackage{amsfonts}       % blackboard math symbols
\usepackage{nicefrac}       % compact symbols for 1/2, etc.
\usepackage{microtype}      % microtypography
\usepackage[most]{tcolorbox}
\usepackage{xspace} 

\usepackage{subcaption}
\usepackage{amsthm}
\usepackage{algorithm} 
\usepackage{algpseudocode}
\usepackage{siunitx}
\usepackage{color , colortbl}
\definecolor{tabhighlight}{rgb}{0.925,1,1}
\usepackage{wrapfig}
\usepackage{multirow}
\usepackage{adjustbox}
\usepackage[table]{xcolor} % 用于行背景色
\definecolor{Gray}{gray}{0.9} % 定义浅灰色
\usepackage{amsmath} 
\usepackage{pifont}
\usepackage{amssymb}
\usepackage{enumitem}
\setlist[itemize]{leftmargin=1.2em}
\usepackage{array}
\newcolumntype{C}[1]{>{\centering\arraybackslash}p{#1}}
\usepackage{tabularx}
\usepackage{array}

\newcolumntype{T}{>{\hsize=0.9\hsize\linewidth=\hsize\centering\arraybackslash}X}
\newcolumntype{M}{>{\hsize=0.8\hsize\linewidth=\hsize\centering\arraybackslash}X}

\usepackage{makecell}

\title{MER-R1: Multimodal Emotion Reasoning via Slow-Fast Thinking Synergy}

\author{%
\normalfont
\begin{tabular}{c}
\textbf{Zhiyuan Han}\textsuperscript{1,2,4}\quad
\textbf{Beier Zhu}\textsuperscript{1,$\dagger$}\quad
\textbf{Wenwen Tong}\textsuperscript{2}\quad
\textbf{Chengwei Qin}\textsuperscript{3}\quad
\textbf{Xinyi Wang}\textsuperscript{1}\\
\textbf{Jiayu Zhang}\textsuperscript{1}\quad
\textbf{Jiangnan Chen}\textsuperscript{2}\quad
\textbf{Hewei Guo}\textsuperscript{2}\quad
\textbf{Dongchuan Ran}\textsuperscript{2}\quad
\textbf{Lewei Lu}\textsuperscript{2}\quad
\textbf{Xun Yang}\textsuperscript{1,$\dagger$} \\[0.3em]
{\normalfont $^{1}$University of Science and Technology of China} \quad
{\normalfont $^{2}$SenseTime Research} \\
{\normalfont $^{3}$The Hong Kong University of Science and Technology (Guangzhou)} \\
{\normalfont $^{4}$Institute of Artificial Intelligence, Hefei Comprehensive National Science Center} \\[0.3em]
{\normalfont\texttt{beier.zhu@ustc.edu.cn}, \quad \texttt{xyang21@ustc.edu.cn}} \\
{\normalfont $^{\dagger}$Corresponding authors.}
\end{tabular}
}

\begin{document}
\maketitle

\newcommand{\zbe}[1]{{\color{blue}#1}}

\def\eg{\emph{e.g.}} 
\def\Eg{\emph{E.g}}
\def\ie{\emph{i.e.}} 
\def\Ie{\emph{I.e}}
\def\cf{\emph{cf} } 
\def\Cf{\emph{Cf}}
\def\etc{\emph{etc}} 
\def\vs{\emph{vs}}
\def\wrt{w.r.t. } 
\def\dof{d.o.f}
\def\iid{i.i.d} 
\def\wolog{w.l.o.g}
\def\etal{\emph{et al}}

\def\ourmethod{\texttt{MER-R1}}
\def\ourea{\texttt{EA}}
\def\ourobj{\texttt{Disentangled Macro Objectives}}
\def\ourreward{\texttt{Slow-Fast Confidence Synergy}}

% ---------- Finding / Observation box ----------

% brighter academic blue, close to the brightness of the green example
\definecolor{findingblue}{HTML}{2B8AC6}

\newcounter{finding}
\newcommand{\finding}[1]{%
    \refstepcounter{finding}%
    % \vspace{0.1em}
    \begin{tcolorbox}[
        enhanced,
        breakable,
        colback=findingblue!6,
        colframe=findingblue!95,
        boxrule=0.75pt,
        arc=4pt,
        boxsep=3pt,
        left=4pt,
        right=4pt,
        top=4pt,
        bottom=4pt,
        before skip=8pt,
        after skip=8pt,
        drop shadow=gray!25!white
    ]
    \noindent\textbf{\textit{Finding~\thefinding:}} #1
    \end{tcolorbox}
}

\newtheorem{definition}{Definition}
\newtheorem{theorem}{Theorem}
\newtheorem{assumption}{Assumption}
\newtheorem{lemma}{Lemma}
\newtheorem{proposition}{Proposition}
\newtheorem{corollary}{Corollary}

\newtheoremstyle{restatedlemma}
  {\topsep}       % Space above
  {\topsep}       % Space below
  {\itshape}      % Body font
  {}              % Indent amount
  {\bfseries}     % Theorem head font
  {.}             % Punctuation after theorem head
  {.5em}          % Space after theorem head
  {\thmname{#1} \thmnumber{#2} (\thmnote{#3})} % Theorem head spec (can be left empty, meaning ‘normal’)

\newtheoremstyle{restatedproposition}
  {\topsep}       % Space above
  {\topsep}       % Space below
  {\itshape}      % Body font
  {}              % Indent amount
  {\bfseries}     % Theorem head font
  {.}             % Punctuation after theorem head
  {.5em}          % Space after theorem head
  {\thmname{#1} \thmnumber{#2} (\thmnote{#3})} % Theorem head spec (can be left empty, meaning ‘normal’)

\theoremstyle{restatedlemma}
\newtheorem*{restatedlemma}{Restated Lemma}

\theoremstyle{restatedproposition}
\newtheorem*{restatedproposition}{Restated Proposition}

\begin{abstract}

We find that explicit reasoning does not necessarily translate into better multimodal emotion recognition (MER) accuracy, even though it makes predictions more interpretable. Specifically, for reasoning-based MLLMs, \textit{fast thinking} by triggering direct answers often outperforms \textit{slow thinking} after deliberative reasoning.
Our empirical  analyses show that fast thinking improves recall with broader and more confident predictions, whereas slow thinking favors precision through conservative filtering of incorrect categories.
Building on these insights, we propose \ourmethod, a reinforcement learning framework that turns slow-fast complementarity into explicit optimization.
Dual-objective disentanglement separates recall and precision into two optimization signals, allowing them to be jointly optimized rather than traded off against each other.
Slow-fast confidence calibration further aligns the final slow-thinking answer with fast-thinking intuition, strengthening correct emotions while suppressing incorrect ones.
In this way, \ourmethod~unifies the recall-oriented intuition of fast thinking with the precision-oriented selectivity of slow thinking. 
We further provide theoretical justification for this synergy, showing that it mitigates variance-induced interference during optimization.
 Extensive experiments on MER-UniBench and MME-Emotion show that \ourmethod~achieves SoTA performance and makes reasoning genuinely benefit emotion recognition.
\end{abstract}
\section{Introduction}
\label{introduction}

Recent advances in Multimodal Large Language Models (MLLMs)~\cite{hurst2024gpt,singh2025openai,xu2025qwen2} are transforming affective computing~\cite{picard2000affective,lian2025mer,shou2025multimodal}, shifting multimodal emotion recognition~\cite{lian2023mer,lian2024mer,poria2019meld,busso2008iemocap}  from fixed-label prediction to \textit{Open-Vocabulary Multimodal Emotion Recognition} (OV-MER)~\cite{lian2025affectgpt,lian2024ov}, where models   recognize diverse and open-ended human emotions. To improve trustworthiness~\cite{lian2023explainable}, recent RLVR-based approaches~\cite{guo2025deepseek,lian2025affectgptr1,yang2025humanomniv2} further equip emotion MLLMs with explicit reasoning chains that ground emotion predictions in visual, acoustic, and textual cues.

Despite this promise, we find that such reasoning does not necessarily lead to stronger emotion recognition. Specifically, we consider two thinking modes for an R1-like emotion reasoning model~\cite{xu2025qwen2,lian2025affectgptr1}: \textbf{(1) Slow thinking} produces an answer after standard reasoning process, and \textbf{(2) Fast thinking} elicits an answer directly without reasoning. Despite well-formed reasoning traces (Figure~\ref{fig:fig1}~(a)), slow thinking consistently underperforms fast
thinking on current OV-MER benchmarks such as MER-UniBench~\cite{lian2025affectgpt,lian2024ov} (Figure~\ref{fig:fig1}~(b)). This reveals a \textit{thinking paradox} in MER: 

\begin{center}
\textit{Reasoning improves interpretability, yet fails to improve recognition accuracy itself.}
\end{center}

To understand this paradox, we analyze fast and slow thinking in Section~\ref{sec:findings} from prediction and confidence perspectives. At the prediction level, fast thinking exhibits stronger System-1 intuition~\cite{li2025system}, producing broader emotion coverage and higher recall, whereas slow thinking makes more concentrated predictions and favors precision. At the confidence level, fast thinking is more confident on correct emotions, while slow thinking more conservatively suppresses incorrect categories but may also lower confidence in correct ones. These findings suggest two \textbf{slow-fast synergy desiderata} for effective MER reasoning: preserving the recall-oriented coverage and correct-category confidence of fast thinking, while retaining the precision-oriented selectivity of slow thinking.

\begin{figure*}[t]
\centering
\includegraphics[width=\linewidth]{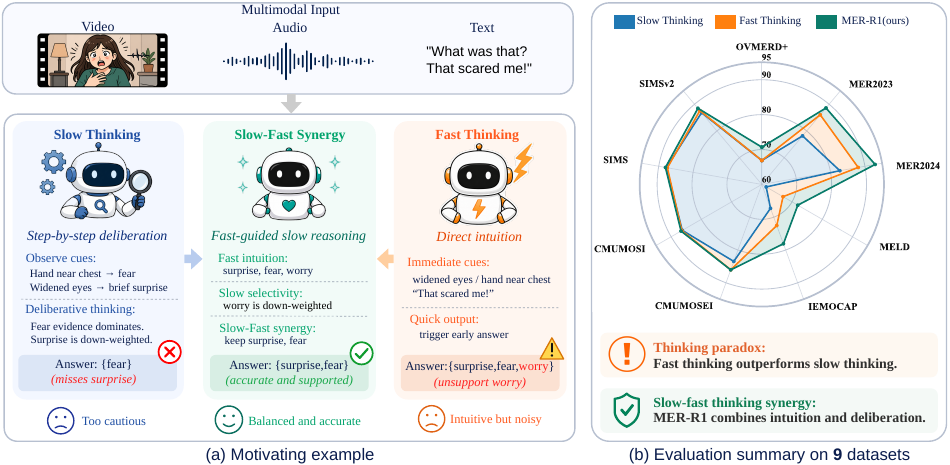}
\captionsetup{skip=6pt} 
\caption{\textbf{Thinking paradox and slow-fast synergy in MER.} (a) Motivating example: slow thinking is cautious but may miss valid emotions, whereas fast thinking improves coverage but introduces noise. Slow-fast synergy combines their complementary strengths. (b) Evaluation summary on 9 datasets: Fast thinking counter-intuitively outperforms slow thinking. }
\label{fig:fig1}
\vspace{-2mm}
\end{figure*}

To realize these desiderata, we propose \ourmethod, a reinforcement learning (RL) framework with two components.
First, \textit{dual-objective disentanglement} separates recall and precision into two optimization targets and preserves them in both reward and advantage spaces.
Unlike prior MER reasoning methods that directly optimize $F_1$ reward~\cite{lian2025affectgptr1}, our design prevents recall and precision interference during advantage normalization.
Our theoretical analysis further shows that standard $F_1$ optimization can be biased toward the objective with larger group-wise variation, whereas the dual-objective advantage establishes a balanced coupling to both recall and precision signals.
Second, \textit{slow-fast confidence calibration} transfers fast thinking's confidence behavior to the final slow-thinking answer while preserving
slow thinking's selectivity. Specifically, we compare category-level confidence between the two thinking modes and calibrate correct and incorrect categories in opposite directions: correct categories are encouraged to retain or exceed fast-thinking confidence, whereas incorrect categories are encouraged to remain suppressed. This enables \ourmethod~to strengthen correct emotions without absorbing the noisy over-coverage of
fast thinking.

\textbf{Our contributions} are three-fold:
\textbf{(1) Findings:} We reveal a \textit{thinking paradox} in MER: for reasoning-based MLLMs, fast thinking can outperform slow thinking. We diagnose this paradox from prediction and confidence perspectives and derive two desiderata for slow-fast synergy.
\textbf{(2) Methodology:} We propose \ourmethod, an RL framework with dual-objective disentanglement and slow-fast confidence calibration, enabling joint recall--precision optimization and confidence calibration over correct and incorrect categories.
\textbf{(3) Performance:} \ourmethod~achieves state-of-the-art results on MER-UniBench and MME-Emotion, and validates that satisfying the two slow-fast synergy desiderata makes reasoning genuinely beneficial for emotion recognition rather than merely improving interpretability.
 
\section{Related work}
\label{sec:relatedwork}

\noindent\textbf{MLLM reasoning.}
Recent multimodal large language models (MLLMs)~\cite{hurst2024gpt,Qwen3-Omni,xu2025qwen2,liu2024deepseek,liu2024visual} have acquired reasoning ability through reinforcement learning with verifiable rewards (RLVR)~\cite{guo2025deepseek,li2025system,11515172}, allowing them to explicitly reason over visual, acoustic, and textual signals during the reasoning process. Building on this capability, recent work~\cite{meng2025mm,wang2025perception,liu2025visual,wang2025vl} seeks to tackle increasingly complex multimodal reasoning tasks. However, recent studies suggest that slow thinking is not always better than fast thinking. VideoAuto-R1~\cite{liu2026videoauto} observes this phenomenon on video understanding tasks and addresses it with an \textit{answer-think-answer} paradigm that rewards both the early and final answers while assigning a larger weight to the latter. However, this design remains at the output level and does not model why fast and slow thinking differ, let alone how to realize synergy between them. In contrast, our work studies their underlying complementarity and turns it into an explicit synergy framework for multimodal reasoning, thereby making slow thinking genuinely more effective.

\noindent\textbf{Multimodal emotion understanding.}
Multimodal emotion understanding has recently shifted from closed-set classification to more open and reasoning-oriented settings~\cite{lian2023explainable,cheng2024emotion,han2025benchmarking,lian2025affectgpt}. Open-Vocabulary Multimodal Emotion Recognition (OV-MER)~\cite{lian2024ov} requires models to predict free-form emotion words beyond a fixed label space, while MME-Emotion~\cite{zhang2025mme} further evaluates both recognition and reasoning ability with LLM-based judges. Recent emotion MLLMs~\cite{zhao2025r1,yang2025humanomniv2,lian2025affectgptr1,rha2026emotion} increasingly introduce reasoning for better interpretability and reinforcement learning in an attempt to further improve emotion recognition. However, whether reasoning itself truly improves emotion recognition remains unclear. We identify and systematically study the \emph{thinking paradox} in multimodal emotion reasoning, where fast thinking can outperform slow thinking on recognition, and resolve it through an explicit slow-fast thinking synergy framework.
\section{Understanding the thinking paradox}
\label{sec:findings}

This section investigates the counter-intuitive finding that fast thinking  outperforms slow thinking in MER tasks (Figure~\ref{fig:fig1}~(b)).
We show that fast thinking improves recall with more diverse and confident predictions, whereas slow thinking is more conservative, favoring precision by suppressing incorrect categories but lowering confidence in correct ones. Under current evaluation, which mainly relies on recall or F1 metrics, the advantage of slow thinking is not fully reflected. We begin by introducing the emotion reasoning formulation and the emotion-wheel-based evaluation protocol.

\subsection{Problem formulation and emotion wheel-based evaluation}

\begin{wrapfigure}{r}{0.33\linewidth}
    \vspace{-1em}
    \centering
    \includegraphics[width=\linewidth]{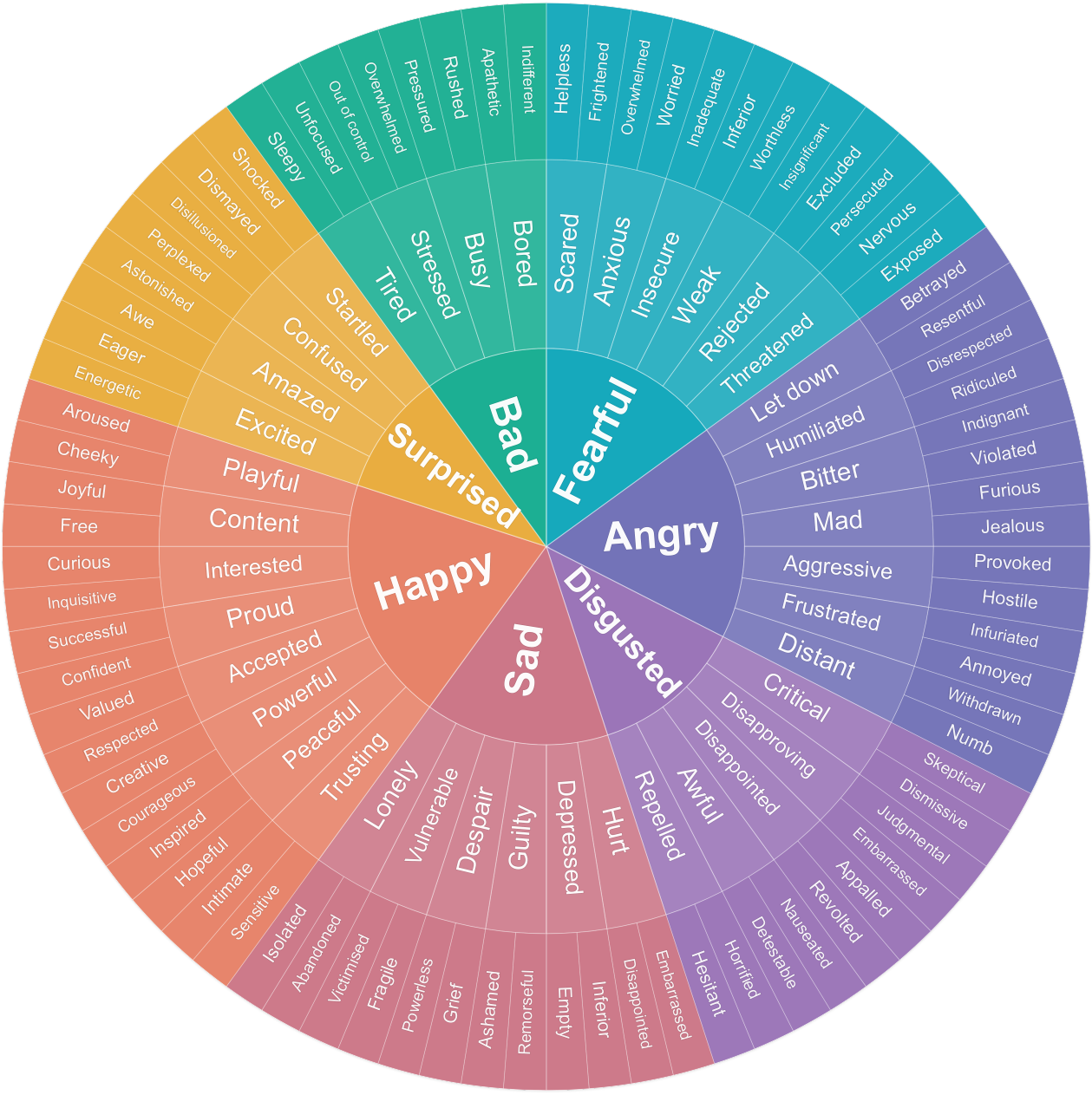}
    \caption{{Emotion wheel.} }
    \label{fig:ew}
    \vspace{-1em}
\end{wrapfigure}

Given a multimodal input $x$, the policy $\pi_\theta$ generates an output
$y=(y^{\mathsf{cot}}, y^{\mathsf{ans}})$, where $y^{\mathsf{cot}}$ denotes the
chain-of-thought reasoning trace and $y^{\mathsf{ans}}$ denotes the final emotion answer.
Slow thinking follows the standard think-then-answer format, while fast thinking is elicited
with an answer-only prompt that directly produces $y^{\mathsf{ans}}$ without reasoning,
\ie, $y^{\mathsf{cot}}=\emptyset$. 
Since both the model answer $y^{\mathsf{ans}}$ and the ground-truth emotions
$y^{\mathsf{gt}}$ may contain multiple free-form emotion words, we follow prior
work~\cite{lian2025affectgpt,lian2024ov} and map them through the emotion wheel
to level-1 categories. For example, as shown in Figure~\ref{fig:ew}, the \textbf{innermost sectors define
level-1 emotion categories}, such as \texttt{angry}, \texttt{happy}, and
\texttt{sad}. Fine-grained emotion words in the outer sectors are then mapped
to these categories; for instance, \texttt{joyful}, \texttt{content}, and
\texttt{cheerful} are mapped to the level-1 category \texttt{happy}.

Let $\hat{\mathcal{Y}}$ and $\mathcal{G}$ denote the predicted and ground-truth
level-1 category sets obtained from $y^{\mathsf{ans}}$ and $y^{\mathsf{gt}}$,
respectively. We evaluate the prediction by set-level matching. Specifically, recall, precision, and $F_1$ are defined as:
$\mathsf{R}
=
\tfrac{|\hat{\mathcal{Y}}\cap \mathcal{G}|}{|\mathcal{G}|}
, 
\mathsf{P}
=
\tfrac{|\hat{\mathcal{Y}}\cap \mathcal{G}|}{|\hat{\mathcal{Y}}|},
\mathsf{F} = \tfrac{2\mathsf{PR}}{\mathsf{P+R}}.
$
Current OV-MER evaluation mainly relies on recall-based metrics~\cite{lian2025affectgpt}
or F1-based metrics~\cite{lian2024ov}. We further examine precision
and category-level confidence to diagnose the thinking paradox.

\begin{figure*}[t]
\centering
\includegraphics[width=\linewidth]{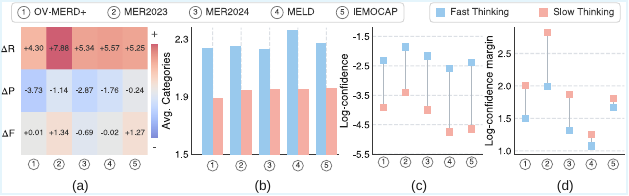}
\captionsetup{skip=6pt} 
\caption{\textbf{Analyses of the thinking paradox on five MER benchmarks.} (a) Performance gaps in recall ($\mathsf{R}$), precision ($\mathsf{P}$), and $F_1$ ($\mathsf{F}$), where $\Delta \mathsf{M}=\mathsf{M}_{\mathsf{fast}}-\mathsf{M}_{\mathsf{slow}}$ for $\mathsf{M}\in\{\mathsf{R},\mathsf{P},\mathsf{F}\}$. (b) Prediction diversity by the average number of emotion categories. (c) Confidence on ground-truth emotion categories. (d) Confidence margin between ground-truth and hard-negative emotion categories.  }
\label{fig:fig2}
\vspace{-2mm}
\end{figure*}

\subsection{Precision--recall trade-off}

In Figure~\ref{fig:fig2}~(a), we report the performance gaps between fast and slow thinking in terms of precision, recall, and F1 on five widely used datasets~\cite{lian2024ov,lian2025mer,lian2023mer,lian2024mer,poria2019meld,busso2008iemocap}. Fast thinking consistently leads in recall, while slow thinking leads in precision; overall, fast thinking retains a slight advantage in F1.
Figure~\ref{fig:fig2}~(b) reports the average number of emotion categories predicted by fast and slow thinking.
The results show a clear difference in prediction diversity: fast thinking covers more emotion categories, whereas slow thinking produces more concentrated predictions.  Together, these two analyses lead to the following finding:

\finding{\textit{Fast thinking improves recall with broader emotion coverage, whereas slow thinking favors precision with more concentrated predictions.}}

Previous benchmarks such as MER-UniBench~\cite{lian2025affectgpt} mainly use hitrate as the evaluation metric, which reflects recall but does not penalize over-prediction (low precision).
This helps explain why fast thinking appears more favorable than slow thinking under current evaluation.

\subsection{Confident vs. conservative predictions}
To compare the confidence behavior, we first aggregate the probability mass of free-form emotion words mapped to the  level-1 category (Definition~\ref{dfn:1}).
We then contrast this confidence between ground-truth and hard-negative categories to define a confidence margin (Definition~\ref{dfn:2}).

\begin{definition}[\textbf{Category-level log-confidence}]\label{dfn:1}
   For a level-1 emotion category $c$, let $\mathcal{V}(c)$ denote the set of emotion words associated with $c$ defined by the emotion wheel~\cite{lian2024ov}. Given a thinking mode $m \in \{\mathsf{fast}, \mathsf{slow}\}$, we define the category-level confidence of $c$ as:
\begin{equation}
\mathsf{conf}^m(c)
=
\log \sum_{e \in \mathcal{V}(c)} p^m(e),
\label{eq:cat_conf}
\end{equation}
where $p^m(e)$ denotes the probability assigned to the first token of emotion word $e$. 
\end{definition}

\begin{definition}[\textbf{Confidence margin}]\label{dfn:2}
Let $\mathcal{G}$ denote the set of ground-truth level-1 categories for a sample, and let $\mathcal{H}^m$ denote the set of top-3 hard-negative categories. We define
$
\mathsf{conf}_{\mathsf{gt}}^m
=
\frac{1}{|\mathcal{G}|}
\sum_{c \in \mathcal{G}} \mathsf{conf}^m(c)
$ and $
\mathsf{conf}_{\mathsf{neg}}^m
=
\frac{1}{|\mathcal{H}^m|}
\sum_{c \in \mathcal{H}^m} \mathsf{conf}^m(c)
\label{eq:cat_conf_stats}
$
as the average confidence assigned to correct and incorrect categories, respectively. Their relative confidence margin is defined as:
\begin{equation}
\gamma^m
=
\mathsf{conf}_{\mathsf{gt}}^m
-
\mathsf{conf}_{\mathsf{neg}}^m.
\label{eq:cat_margin}
\end{equation}
\end{definition}

Figure~\ref{fig:fig2}(c) reports the category-level log-confidence on ground-truth emotion categories, showing that fast thinking consistently assigns higher confidence to correct emotions.
Figure~\ref{fig:fig2}(d) reports the confidence margin between ground-truth and hard-negative categories, showing that slow thinking yields larger margins by suppressing hard negatives.
Together, these results lead to Finding 2:

\finding{\textit{Fast thinking is more confident on correct emotions, while slow thinking is more conservative, better suppressing incorrect categories but under-confidence correct ones.}}

\subsection{Desiderata for slow-fast thinking synergy}
\label{sec:desiderata}

These observations suggest that an effective emotion reasoning model should not
simply prefer one thinking mode over the other. Instead, it should satisfy two
desiderata. 

\begin{itemize}
    \item \textbf{Desideratum I}: At the prediction level, the model should preserve the
recall-oriented coverage of fast thinking while maintaining the precision-oriented
selectivity of slow thinking. 
    \item \textbf{Desideratum II}: At the confidence level, the model should preserve
fast thinking's strong confidence on correct categories, while retaining slow
thinking's suppression of incorrect ones.
\end{itemize}

These desiderata motivate our slow-fast thinking synergy framework in
Section~\ref{sec:method}.
\section{Method}
\label{sec:method}

We propose \ourmethod, an RL framework for slow-fast thinking synergy.
Section~\ref{sec:baseline} introduces the GRPO baseline and its limitations. Section~\ref{sec:method1} proposes dual-objective disentanglement to preserve the recall-oriented coverage and the precision-oriented selectivity in the advantage
space (Desideratum I). Section~\ref{sec:method2} introduces fast-thinking
confidence calibration to strengthen correct categories while suppressing incorrect ones (Desideratum II). The overall pipeline is illustrated in Figure~\ref{fig:fig3}.

\subsection{Preliminary on GRPO baseline}\label{sec:baseline}

Reasoning-oriented MLLMs are commonly trained with a two-stage recipe: a
supervised fine-tuning (SFT) stage,
followed by GRPO-style RL.
Given input $x$, GRPO samples $K$ outputs $\{y_i\}_{i=1}^{K}$ from
$\pi_{\theta_{\mathsf{old}}}$, computes rewards $\{R_i\}_{i=1}^{K}$, normalizes
them into relative advantages $A_i=\tfrac{R_i-\mu}{\sigma}$, and optimizes:
\begin{equation}
\mathcal{L}_{\mathsf{GRPO}}(\theta)
=
-\frac{1}{K}\sum_{i=1}^{K}
\min\!\big(r_i A_i,\; \mathrm{clip}(r_i,1-\eta,1+\eta)A_i\big)
+
\beta D_{\mathrm{KL}}(\pi_\theta \,\|\, \pi_{\mathsf{ref}}),
\label{eq:grpo_prelim}
\end{equation}
where $r_i=\tfrac{\pi_\theta(y_i\mid x)}{\pi_{\theta_{\mathsf{old}}}(y_i\mid x)}$ is the importance ratio, and $\beta \ge 0$ controls the KL regularization strength.
As a typical reward design~\cite{lian2025affectgptr1}, the baseline combines an
emotion-wheel-based $F_1$ reward with a format reward:
$ 
R_{\mathsf{base}}
=
R_{\mathsf{F}}
+ R_{\mathsf{fmt}},
\label{eq:base_reward}
$
where  $R_{\mathsf{fmt}}$ encourages  \texttt{<think>} then \texttt{<answer>} structure.

However, the GRPO baseline has two limitations:
the $F_1$-style reward entangles recall and precision, allowing the model to trade off
one for the other, and the training objective does not explicitly preserve the
recall-oriented intuition of fast thinking.  

\begin{figure*}[t]
\centering
\includegraphics[width=\linewidth]{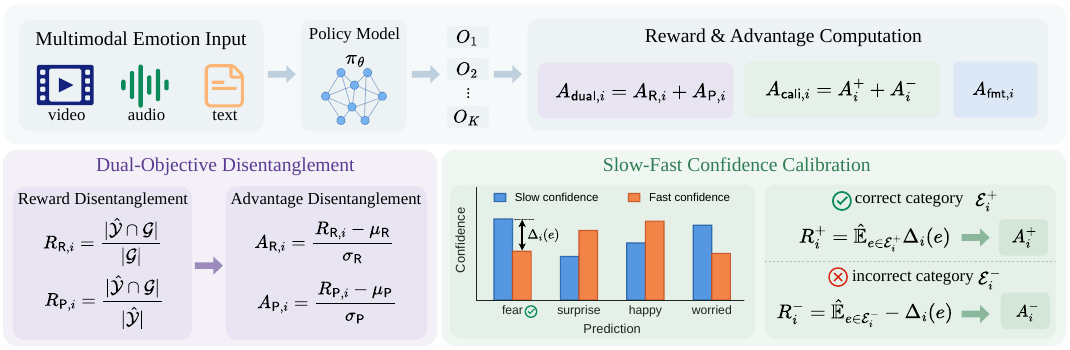}
\captionsetup{skip=6pt} 
\caption{\textbf{Framework of \ourmethod.} Our
\ourmethod~consists of dual-objective disentanglement for Desideratum I and slow-fast confidence calibration for Desideratum II.}
\label{fig:fig3}
\vspace{-2mm}
\end{figure*}

\subsection{Dual-objective disentanglement for Desideratum I}
\label{sec:method1}
To avoid the ambiguous trade-off induced by a single $F_1$-style reward, we explicitly separate the two optimization goals behind emotion recognition: covering correct categories and avoiding spurious ones. We first decompose the task reward into recall and precision rewards, and then compute their group-normalized advantages separately as the optimization signals.

\noindent\textbf{Reward-level disentanglement.}
We decompose $F_1$ reward into recall and precision rewards:
\begin{equation}
R_{\mathsf{R}}
=
\tfrac{|\hat{\mathcal{Y}}\cap \mathcal{G}|}{|\mathcal{G}|},
\qquad
R_{\mathsf{P}}
=
\tfrac{|\hat{\mathcal{Y}}\cap \mathcal{G}|}{|\hat{\mathcal{Y}}|}.
\label{eq:macro_rec_pre}
\end{equation}
This decomposition makes the hidden bi-objective structure of $F_1$ explicit:
$R_{\mathsf{R}}$ promotes coverage of correct categories, while $R_{\mathsf{P}}$
penalizes spurious predictions. Since different recall--precision trade-offs can yield the same $F_1$ value,
separating $R_{\mathsf{R}}$ and $R_{\mathsf{P}}$ avoids an ambiguous optimization target.

\noindent\textbf{Advantage-level disentanglement.}
Given the two rewards $R_{\mathsf{R}}$ and $R_{\mathsf{P}}$, a straightforward GRPO pipeline would first aggregate them into a scalar reward and then apply group normalization to obtain the advantage.
However, when the two rewards have different group-wise variances, the combined scalar reward is dominated by the higher-variance objective, causing the lower-variance objective to be underrepresented in the final advantage.
We therefore first compute group-normalized advantages for $R_{\mathsf{R}}$ and $R_{\mathsf{P}}$ separately then compute the final advantage $A_{\mathsf{dual}}$:
\begin{equation}
A_{\mathsf{R},i}
=
\tfrac{R_{\mathsf{R},i}-\mu_{\mathsf{R}}}{\sigma_{\mathsf{R}}},
\qquad
A_{\mathsf{P},i}
=
\tfrac{R_{\mathsf{P},i}-\mu_{\mathsf{P}}}{\sigma_{\mathsf{P}}}, \qquad A_{\mathsf{dual},i}
=
  A_{\mathsf{R},i}
+
 A_{\mathsf{P},i}.
\label{eq:macro_adv}
\end{equation}

\noindent\textbf{Theoretical justification.}
The following proposition shows that standard $F_1$ optimization can be biased toward the objective with larger normalized variation, whereas our dual-objective disentanglement removes this correlation imbalance (proof in Appendix~\ref{app:theory1}). 

\begin{proposition} 
\label{prop:proposition1}
Let $\mu_{\mathsf{R}},\mu_{\mathsf{P}}$ and
$\sigma_{\mathsf{R}},\sigma_{\mathsf{P}}$ denote the group means and standard
deviations of the recall and precision rewards, and let
$\tau=\frac{\sigma_\mathsf{R}/\mu_\mathsf{R}^2}{\sigma_\mathsf{P}/\mu_\mathsf{P}^2}$ be the normalized group-wise variance ratio.
The correlations between each reward and the standard $F_1$ advantage satisfy
\begin{equation}
\begin{cases}
\mathsf{corr}(R_{\mathsf{R}},A_{\mathsf{F}})
>
\mathsf{corr}(R_{\mathsf{P}},A_{\mathsf{F}}),
& \text{if } \tau>1, \\[1mm]
\mathsf{corr}(R_{\mathsf{R}},A_{\mathsf{F}})
\leq
\mathsf{corr}(R_{\mathsf{P}},A_{\mathsf{F}}),
& \text{if } \tau\leq 1.
\end{cases}
\label{eq:f1_corr_bias}
\end{equation}
In contrast, our dual-objective advantage gives \textbf{balanced correlations}:
\begin{equation}
\mathsf{corr}(R_{\mathsf{R}},A_{\mathsf{dual}})
=
\mathsf{corr}(R_{\mathsf{P}},A_{\mathsf{dual}}).
\end{equation}
\end{proposition}

\subsection{Slow-fast confidence calibration for Desideratum II}
\label{sec:method2}
For Desideratum II, we propose slow-fast confidence calibration, which guides the final slow-thinking answer to preserve fast thinking's confidence on correct categories while retaining its suppression of incorrect ones. 
For each response $y_i=(y_i^{\mathsf{cot}}, y_i^{\mathsf{ans}})$, we construct two answer-generation distributions:
\begin{equation}
p_{i}^{\mathsf{slow}}(e) = \pi_\theta(\cdot \mid x, y_i^{\mathsf{cot}}, y_{i,<e}^{\mathsf{ans}}),
\qquad
p_{i}^{\mathsf{fast}}(e) = \pi_\theta(\cdot \mid x, y_{i,<e}^{\mathsf{ans}}).
\label{eq:slow_fast_forward}
\end{equation}
For each generated emotion word $e \in y_i^{\mathsf{ans}}$, we map it to its
level-1 category $c(e)$ and compute its slow-fast confidence gap using
Eq.~\eqref{eq:cat_conf}:
\begin{equation}
\Delta_i(e)
=
\mathsf{conf}_i^{\mathsf{slow}}(c(e))
-
\mathsf{conf}_i^{\mathsf{fast}}(c(e)).
\label{eq:delta_token}
\end{equation}

We split the emotion words parsed from $y_i^{\mathsf{ans}}$ into correct ($\mathcal{E}_i^{+}$) and incorrect ($\mathcal{E}_i^{-}$) sets:
\begin{equation}
\mathcal{E}_i^{+}
=
\{e \in y_i^{\mathsf{ans}} \mid c(e)\in \mathcal{G}\},
\qquad
\mathcal{E}_i^{-}
=
\{e \in y_i^{\mathsf{ans}} \mid c(e)\notin \mathcal{G}\}.
\label{eq:emotion_word_partition}
\end{equation}

We use the \textit{slow-fast confidence gap} to calibrate the answer in two opposite
directions. For emotion words mapped to correct categories, we reward a positive
gap so that slow thinking preserves or exceeds the strong confidence of fast
thinking. For incorrect emotion words, we reward a negative
gap so that slow thinking retains its conservative suppression of spurious
predictions:
\begin{equation}
R_i^{+}
=
\frac{1}{|\mathcal{E}_i^{+}|}
\sum_{e\in\mathcal{E}_i^{+}} \Delta_i(e),
\qquad
R_i^{-}
=
\frac{1}{|\mathcal{E}_i^{-}|}
\sum_{e\in\mathcal{E}_i^{-}} -\Delta_i(e).
\label{eq:confidence_rewards}
\end{equation}
 
Following the same disentanglement strategy, we first compute group-normalized
advantages for $R_i^+$ and $R_i^-$ separately, denoted as $A_i^+$ and $A_i^-$,
and then combine them as $A_{\mathsf{cali},i}=A_i^+ + A_i^-$.
The overall advantage is combined as:
\begin{equation}
A_i
=
A_{\mathsf{dual},i}
+
\lambda_{\mathsf{cali}} A_{\mathsf{cali},i}
+
\lambda_{\mathsf{fmt}} A_{\mathsf{fmt},i}.
\label{eq:final_adv}
\end{equation}
Finally, we substitute $A_i$ into Eq.~\eqref{eq:grpo_prelim} for policy optimization.

\section{Experiments}
\label{sec:experiments}

\begin{table*}[t]
\centering

\fontsize{7.5}{10}\selectfont
\setlength{\tabcolsep}{4.8pt}

\caption{\textbf{Main results on MER-UniBench.} \ourmethod~achieves the best overall performance under the official benchmark protocol. Entries marked with $\dagger$ are our reproduced results.}
\label{tab:main}

\begin{tabular}{l c @{\hspace{3pt}} cccc @{\hspace{3pt}} cccc @{\hspace{4pt}} c}
    \toprule

    \multirow{2}{*}{\textbf{Model}} 
    & \multicolumn{1}{c}{\textbf{Fine (F1)}} 
    & \multicolumn{4}{c}{\textbf{Basic Emotion (Hitrate)}} 
    & \multicolumn{4}{c}{\textbf{Sentiment Analysis (WAF)}} 
    & \multirow{2}{*}{\textbf{Mean}} \\
    \cmidrule(lr){2-2} \cmidrule(lr){3-6} \cmidrule(lr){7-10}
    & OV-MERD+ & MER23 & MER24 & MELD & IEMOCAP & MOSI & MOSEI & SIMS & SIMSv2 & \\
    \midrule

    Qwen-Audio~\cite{chu2023qwen}   & 32.36 & 41.85 & 31.61 & 49.09 & 35.47 & 70.09 & 46.90 & 70.73 & 65.26 & 49.26 \\
    SALMONN~\cite{tang2023salmonn}  & 45.00 & 55.53 & 45.38 & 45.62 & 46.84 & 81.00 & 67.03 & 68.69 & 65.93 & 57.89 \\

    VideoChat2~\cite{li2024mvbench} & 39.21 & 33.67 & 54.50 & 36.64 & 48.70 & 66.84 & 54.32 & 69.49 & 70.66 & 52.67 \\
    LLaMA-VID~\cite{li2024llama}    & 45.01 & 50.72 & 57.60 & 42.75 & 46.02 & 61.78 & 63.89 & 69.35 & 67.48 & 56.07 \\
    Chat-UniVi~\cite{jin2024chat}   & 48.00 & 57.62 & 65.67 & 45.61 & 52.37 & 54.53 & 63.18 & 68.15 & 66.36 & 57.94 \\
    mPLUG-Owl~\cite{ye2023mplug}    & 48.18 & 56.86 & 59.89 & 49.11 & 55.54 & 72.40 & 72.91 & 72.13 & 75.00 & 62.45 \\

    PandaGPT~\cite{su2023pandagpt}  & 37.12 & 40.21 & 51.89 & 37.88 & 44.04 & 61.92 & 67.61 & 68.38 & 67.23 & 52.92 \\
    R1-Omni~\cite{zhao2025r1}       & 55.24 & 64.17 & 67.43 & 43.20 & 51.58 & 58.02 & 56.48 & 71.82 & 68.58 & 59.61 \\
    Emotion-LLaMA~\cite{cheng2024emotion} & 52.97 & 59.38 & 73.62 & 46.76 & 55.47 & 66.13 & 67.66 & 78.32 & 77.23 & 64.17 \\

    AffectGPT~\cite{lian2025affectgpt} & 62.52 & 78.54 & 78.80 & 55.65 & 60.54 & 81.30 & 80.90 & \textbf{88.49} & 86.18 & 74.77 \\
    AffectGPT-R1~\cite{lian2025affectgptr1} & 68.39 & 84.51 & \textbf{93.13} & 66.71 & 74.26 & 79.65 & 80.18 & 87.26 & 85.75 & 79.98 \\

    AffectGPT$^{\dagger}$  & 59.89 & 73.69 & 76.90 & 52.09 & 57.74 & 78.30 & 78.51 & 85.28 & 85.31 & 71.96 \\
    AffectGPT-R1$^{\dagger}$  & 64.08 & 81.88 & 80.89 & 57.53 & 65.71 & 80.29 & 80.64 & 85.70 & 83.75 & 75.60 \\

    \textit{Baseline} & 66.86 & 78.16 & 82.70 & 61.45 & 67.28 & 83.48 & 86.31 & 87.78 & 86.84 & 77.87 \\

    \rowcolor{tabhighlight}
    \textbf{\ourmethod} 
    & \textbf{70.68} & \textbf{88.60} & 92.96 & \textbf{71.90} & \textbf{78.10} 
    & \textbf{86.06} & \textbf{86.73} & 88.04 & \textbf{88.48} & \textbf{83.50} \\
    \bottomrule
\end{tabular}
\end{table*}

\begin{table*}[t]
\centering

\fontsize{7.5}{10}\selectfont
\setlength{\tabcolsep}{2.4pt}

\caption{\textbf{Main results on MME-Emotion.} We report the CoT score for each task and all three metrics for the overall average. \ourmethod~achieves the best overall performance.}
\label{tab:mme_emotion_main}

\begin{tabular}{l @{\hspace{4pt}} c c c c c c *{2}{>{\centering\arraybackslash}m{0.8cm}}@{} @{\hspace{8pt}} *{3}{>{\centering\arraybackslash}m{0.8cm}}}
    \toprule

    \multirow{2}{*}{\textbf{Methods}}
    & \multirow{2}{*}{\textbf{ER-Lab}}
    & \multirow{2}{*}{\textbf{ER-Wild}}
    & \multirow{2}{*}{\textbf{FG-ER}}
    & \multirow{2}{*}{\textbf{FG-SA}}
    
    & \multirow{2}{*}{\textbf{ML-ER}}
    & \multirow{2}{*}{\textbf{Noise-ER}}
    & \multirow{2}{*}{\textbf{IR}}
    & \multirow{2}{*}{\textbf{SA}}
    & \multicolumn{3}{c}{\textbf{Mean}} \\
    \cmidrule(lr){10-12}
    &  &  &  &  &  &  &  &  & \textbf{Rec.} & \textbf{Rea.} & \textbf{CoT} \\
    \midrule

    PandaGPT~\cite{su2023pandagpt}
    & 28.5 & 23.2 & 25.2 & 41.4  & 28.9 & 26.3 & 29.2 & 37.7 & 21.6  & 38.4 & 30.0 \\

    Emotion-LLaMA~\cite{cheng2024emotion}
    & 31.7 & 19.7 & 31.7 & 36.4  & 29.5 & 49.3 & 27.3 & 40.7 & 21.1 & 42.4 & 31.8 \\

    AffectGPT~\cite{lian2025affectgpt}
    & 35.0 & 32.8 & 32.6 & \textbf{44.2}  & 31.3 & 50.7 & 30.0 & 47.9 & 19.7 & 57.4 & 38.6 \\

    \textit{Baseline}
    & 47.0 & 40.8 & 35.0 & 35.6  & 37.7 & 59.0 & 32.4 & 58.7 & 27.9 & 62.7 & 45.3 \\

    VideoAuto-R1~\cite{liu2026videoauto}
    & 49.4 & 44.4 & \textbf{40.2} & 33.1  & 37.3 & 64.8  & 32.8 & 56.7 & 34.5 & 58.8 & 46.6 \\

    \rowcolor{tabhighlight}
    \textbf{\ourmethod}
    & \textbf{57.3} & \textbf{49.6} & 38.7 & 37.2 & \textbf{42.2} & \textbf{72.1} & \textbf{33.9} & \textbf{62.9} & \textbf{38.4} & \textbf{64.6} & \textbf{51.5} \\
    \bottomrule
\end{tabular}
\vspace{-2mm}
\end{table*}

\begin{table*}[t]
\centering

\fontsize{7.5}{10}\selectfont
\setlength{\tabcolsep}{3.2pt}

\caption{\textbf{Slow Thinking vs.\ Fast Thinking on MER-UniBench under unified F1 metrics.} The four basic emotion recognition datasets are re-evaluated by F1. \textbf{R-Mean} averages the two emotion recognition tasks, \ie, fine-grained and basic emotion recognition. \ourmethod~ makes the final slow-thinking answer outperform its fast-thinking counterpart.}
\label{tab:main_fast}

\begin{tabular}{l l c @{\hspace{2pt}} cccc @{\hspace{2pt}} cccc @{\hspace{3pt}} c @{\hspace{2pt}} c}
    \toprule

    \multirow{2}{*}{\textbf{Model}} 
    & \multirow{2}{*}{\textbf{Mode}}
    & \multicolumn{1}{c}{\textbf{Fine (F1)}} 
    & \multicolumn{4}{c}{\textbf{Basic Emotion (F1)}} 
    & \multicolumn{4}{c}{\textbf{Sentiment Analysis (WAF)}} 
    & \multirow{2}{*}{\textbf{R-Mean}} 
    & \multirow{2}{*}{\textbf{Mean}} \\
    \cmidrule(lr){3-3} \cmidrule(lr){4-7} \cmidrule(lr){8-11}
    & & OV-MERD+ & MER23 & MER24 & MELD & IEMOCAP & MOSI & MOSEI & SIMS & SIMSv2 & & \\
    \midrule

    \multirow{2}{*}{\textit{Baseline}}
    & Fast
    & \textbf{66.87} & \textbf{61.92} & 64.37 & 46.18 & \textbf{53.13} 
    & \textbf{86.03} & \textbf{86.47} & 87.48 & \textbf{87.60} 
    & \textbf{58.49} & \textbf{71.11} \\
    
    & Slow
    & 66.86 & 60.58 & \textbf{65.06} & \textbf{46.20} & 51.86 
    & 83.48 & 86.31 & \textbf{87.78} & 86.84 
    & 58.11 & 70.52 \\

    \specialrule{\cmidrulewidth}{1pt}{1pt}

    \multirow{2}{*}{VideoAuto-R1}
    & Fast
    & \textbf{69.67} & \textbf{65.09} & \textbf{67.38} & \textbf{46.77} & \textbf{56.33}
    & \textbf{86.21} & \textbf{86.05} & 86.20 & \textbf{87.93} 
    & \textbf{61.05} & \textbf{72.40} \\
    
    & Slow
    & 69.49 & 64.20 & 67.20 & 46.36 & 55.32 
    & 85.75 & 85.86 & \textbf{86.27} & 87.83  
    & 60.51 & 72.03 \\

    \specialrule{\cmidrulewidth}{1pt}{1pt}

    \multirow{2}{*}{\textbf{\ourmethod}}
    & Fast
    & 70.10 & 64.41 & 67.21 & 46.89 & 55.71 
    & 85.30 & \textbf{87.43} & \textbf{88.57} & 87.49 
    & 60.86 & 72.56 \\

    & \cellcolor{tabhighlight}Slow
    & \cellcolor{tabhighlight}\textbf{70.68} 
    & \cellcolor{tabhighlight}\textbf{65.43} 
    & \cellcolor{tabhighlight}\textbf{68.46} 
    & \cellcolor{tabhighlight}\textbf{48.70} 
    & \cellcolor{tabhighlight}\textbf{55.73} 
    & \cellcolor{tabhighlight}\textbf{86.06}
    & \cellcolor{tabhighlight}86.73 
    & \cellcolor{tabhighlight}88.04 
    & \cellcolor{tabhighlight}\textbf{88.48} 
    & \cellcolor{tabhighlight}\textbf{61.80} 
    & \cellcolor{tabhighlight}\textbf{73.14}\\

    \bottomrule
\end{tabular}
\vspace{-2mm}
\end{table*}

\begin{figure*}[t]
\centering
\includegraphics[width=\linewidth]{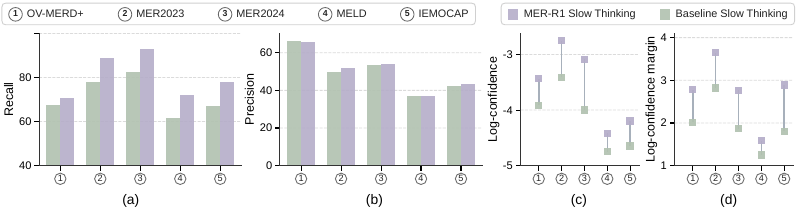}
\captionsetup{skip=6pt} 
\caption{\textbf{Effect of slow-fast thinking synergy}
(a) Recall. 
(b) Precision. 
(c) Confidence on ground-truth categories. 
(d) Confidence margin. 
\ourmethod~improves recall, preserves precision, boosts confidence on ground-truth categories, and strengthens confidence discrimination.}
\label{fig:fig4}
\vspace{-2mm}
\end{figure*}

\begin{table}[t]
\centering
\fontsize{7.5}{10}\selectfont
\setlength{\tabcolsep}{6.0pt}

\caption{\textbf{Main ablation study on MER-UniBench.} \textbf{RD}, \textbf{AD}, and \textbf{SFCC} denote Reward Disentanglement, Advantage Disentanglement, and Slow-Fast Confidence Calibration, respectively. \textbf{Mean\(_{\text{official}}\)} uses the original bench metrics, while \textbf{Mean\(_{\text{f1}}\)} uses F1 for the basic emotion recognition task.}
\label{tab:ablation_main}

\begin{tabular}{l c c c @{\hspace{8pt}} c c c c @{\hspace{4pt}} c c}
    \toprule
    \textbf{Variant}
    & \textbf{RD} & \textbf{AD} & \textbf{SFCC}
    & \textbf{Fine (F1)}
    & \textbf{Basic (Hitrate)}
    & \textbf{Basic (F1)}
    & \textbf{Sentiment (WAF)}
    & \textbf{Mean\(_{\text{official}}\)}
    & \textbf{Mean\(_{\text{f1}}\)} \\
    \midrule

    \textit{Baseline}
    &  &  & 
    & 66.86 & 72.40 & 55.93 & 86.10 & 77.87 & 70.55 \\

    Ex1 
    & \ding{51} &  & 
    & 68.06 & 77.50 & 57.18 & 85.89 & 80.18 & 71.15 \\

    Ex2 
    & \ding{51} & \ding{51} & 
    & 69.77 & 82.40 & 58.55 & 85.11 & 82.20 & 71.60 \\

    \rowcolor{tabhighlight}
    Ex3
    & \ding{51} & \ding{51} & \ding{51}
    & \textbf{70.68} & \textbf{82.90} & \textbf{59.58} & \textbf{87.32} & \textbf{83.50} & \textbf{73.15} \\
    \bottomrule
\end{tabular}
\vspace{-2mm}
\end{table}

\begin{table*}[t]
\centering
\fontsize{7.5}{10}\selectfont
\setlength{\tabcolsep}{8pt}

\caption{\textbf{Ablation on slow-fast confidence calibration.} The first block compares different calibration formulations, including confidence granularity, and the role of the precision calibration term. The second block compares different optimization spaces.}
\label{tab:ablation_micro_opt}

\begin{tabular}{l @{\hspace{4pt}} c c c c @{\hspace{4pt}} c c}
    \toprule
    \textbf{Variant}
    & \textbf{Fine (F1)}
    & \textbf{Basic (Hitrate)}
    & \textbf{Basic (F1)}
    & \textbf{Sentiment (WAF)}
    & \textbf{Mean\(_{\text{official}}\)}
    & \textbf{Mean\(_{\text{f1}}\)} \\
    \midrule

    \textit{Baseline}
    & 66.86 & 72.40 & 55.93 & 86.10 & 77.87 & 70.55 \\
    
    \midrule
    \multicolumn{7}{c}{\textbf{Study on calibration formulation}} \\
    \midrule

    A1: word-level confidence
    & 69.07 & 80.67 & 58.89 & 86.14 & 81.81 & 72.13 \\
    
    % A2: answer-only credit
    % & 69.40 & 81.11 & 59.25 & 86.40 & 82.16 & 72.45 \\

    A2: w/o precision calibration
    & 69.48 & 82.51 & 59.09 & 86.58 & 82.87 & 72.46 \\

    \midrule
    \multicolumn{7}{c}{\textbf{Study on optimization space}} \\
    \midrule

    B1: reward-space calibration
    & 69.03 & 81.83 & 59.06 & 85.19 & 81.91 & 71.78 \\

    B2: mixed-space calibration
    & 67.81 & \textbf{83.62} & 58.44 & 86.42 & 83.11 & 71.91 \\

    \midrule
    \rowcolor{tabhighlight}
    \textbf{\ourmethod}
     & \textbf{70.68} & 82.90 & \textbf{59.58} & \textbf{87.32} & \textbf{83.50} & \textbf{73.15} \\
    \bottomrule
\end{tabular}
\vspace{-2mm}
\end{table*}

\noindent\textbf{Datasets and metrics.}
We evaluate \ourmethod~on two multimodal emotion benchmarks: MER-UniBench~\cite{lian2025affectgpt} and MME-Emotion~\cite{zhang2025mme}.
On \textbf{MER-UniBench}, we follow its three-task setting with 9 datasets: fine-grained emotion recognition on OV-MERD+~\cite{lian2024ov}, evaluated by the \textit{F1-score} metric; basic emotion recognition on MER23~\cite{lian2023mer}, MER24~\cite{lian2024mer}, MELD~\cite{poria2019meld}, and IEMOCAP~\cite{busso2008iemocap}, evaluated by \textit{Hitrate}; and sentiment analysis on MOSI~\cite{zadeh2016mosi}, MOSEI~\cite{zadeh2018multimodal}, SIMS~\cite{yu2020ch}, and SIMSv2~\cite{liu2022make}, evaluated by \textit{Weighted Average F1-score (WAF)}.
Since \textit{Hitrate} mainly reflects recall and does not penalize over-prediction, we re-evaluate the four basic emotion recognition datasets using \textit{F1-score}.
We further evaluate on \textbf{MME-Emotion}, which covers 8 tasks:
lab, wild, noisy, fine-grained, and multi-label emotion recognition
(ER-Lab, ER-Wild, Noise-ER, FG-ER, ML-ER), together with sentiment analysis (SA),
fine-grained sentiment analysis (FG-SA), and intent recognition (IR).
MME-Emotion reports three metrics: \textit{Recognition} for emotion recognition,
\textit{Reasoning} for emotion reasoning, and \textit{CoT} as their average.
Since the original benchmark used GPT-4o as the evaluator, which is no longer available,
we instead use \texttt{gemini-3.1-flash-lite-preview}. More details are in Appendix~\ref{appendix:task-details}. 

\noindent\textbf{Implementation details.}
We use Qwen2.5-Omni~\cite{xu2025qwen2} as the backbone and train it on
MER-Caption+~\cite{lian2025affectgpt} with a two-stage pipeline.
We first perform SFT on 5k randomly sampled instances for 2 epochs with a
learning rate of $2\times 10^{-5}$, and then apply RL on the remaining data for 1 epoch with a learning rate of $2\times 10^{-6}$, sampling 4 responses per prompt.
We set both $\lambda_{\mathsf{cali}}$ and $\lambda_{\mathsf{fmt}}$ to $0.1$.
Experiments are conducted on 16 NVIDIA H100 GPUs. More details are in Appendix~\ref{app:training_details}.

\subsection{Main results}
\label{sec:main_mer}

\noindent\textbf{SoTA performance on MER-UniBench.}
Table~\ref{tab:main} reports the main results under the official MER-UniBench protocol.
\ourmethod~achieves a new state-of-the-art mean score of 83.50, outperforming the strongest reproduced baseline by 5.63 points.
The improvement is consistent across fine-grained emotion recognition, basic emotion recognition, and sentiment analysis.
Notably, \ourmethod~improves OV-MERD+ from 66.86 to 70.68, obtains large gains on MER23, MELD, and IEMOCAP, and achieves the best results on MOSI, MOSEI, and SIMSv2.
These results demonstrate the effectiveness of slow-fast thinking synergy across diverse multimodal emotion tasks.

\noindent\textbf{SoTA performance on MME-Emotion.}
Table~\ref{tab:mme_emotion_main} reports the results on MME-Emotion.
\ourmethod~achieves the best overall CoT mean of 51.5, outperforming all compared methods on this broader benchmark.
Compared with the baseline, \ourmethod~substantially improves the overall Recognition score from 27.9 to 38.4, with clear gains on ER-Lab, ER-Wild, ML-ER, Noise-ER, and SA.
Meanwhile, the overall Reasoning score also increases from 62.7 to 64.6, indicating that the improvement does not come at the cost of reasoning quality.
These results show that slow-fast thinking synergy generalizes beyond MER-UniBench and consistently strengthens both emotion recognition and reasoning.

\noindent\textbf{Fulfilling the Desiderata.}
\ourmethod~fulfills the Desiderata defined in Section~\ref{sec:desiderata}.
Table~\ref{tab:main_fast} compares slow and fast thinking under a unified F1-based evaluation, where the four basic emotion recognition datasets are re-evaluated by F1 rather than Hitrate.
Under this setting, the baseline still exhibits the thinking paradox: its final slow-thinking answer is worse than its fast-thinking counterpart.
VideoAuto-R1~\cite{liu2026videoauto} trains an \textit{answer-think-answer} format with separate answer rewards, but still leaves the final answer weaker than the early answer on emotion recognition, with its R-Mean decreasing from 61.05 to 60.51, as it lacks a mechanism for slow-fast complementarity.
In contrast, \ourmethod~successfully reverses this trend: \textbf{the slow-thinking answer outperforms fast thinking on all five recognition datasets}, improving the R-Mean from 60.86 to 61.80 and demonstrating that \ourmethod~turns slow thinking into stronger final recognition.

Figure~\ref{fig:fig4} further explains how \ourmethod~satisfies the two desiderata.
At the prediction level, Figure~\ref{fig:fig4}(a) and (b) show that \ourmethod~improves recall over the baseline while preserving precision.
At the confidence level, Figure~\ref{fig:fig4}(c) and (d) show that \ourmethod~assigns higher confidence to ground-truth categories and enlarges the margin between ground-truth and hard-negative categories.
Together with Table~\ref{tab:main_fast}, these results demonstrate that \ourmethod~fulfills the desiderata, turning slow-fast complementarity into a stronger final slow-thinking answer.

\subsection{Ablation studies and analyses}
\label{sec:ablation}

\noindent\textbf{Main ablation study.}
Table~\ref{tab:ablation_main} presents the main ablation study on MER-UniBench by progressively adding Reward Disentanglement (RD), Advantage Disentanglement (AD), and Slow-Fast Confidence Calibration (SFCC).
RD already brings clear improvements over the baseline, and AD further boosts both fine-grained and basic emotion recognition, showing the benefit of preserving separated optimization signals in the advantage space.
Adding SFCC yields the best performance across all metrics, improving Mean\(_{\text{official}}\) from 77.87 to 83.50 and Mean\(_{\text{f1}}\) from 70.55 to 73.15.
These results show that all three components contribute to the final gain of \ourmethod.

\noindent\textbf{Ablation on slow-fast confidence calibration.}
Table~\ref{tab:ablation_micro_opt} ablates the design of Slow-Fast Confidence Calibration.
A1 uses word-level confidence instead of category-level confidence and performs worse, showing the importance of calibrating in the emotion-wheel category space.
A2 removes precision calibration and is also inferior to the full model, indicating that recall-side calibration alone cannot preserve slow-thinking selectivity.
For optimization space, B1 applies calibration in the reward space, while B2 adds corresponding dual and calibration rewards before advantage normalization.
Both are weaker than \ourmethod~in overall performance, demonstrating that SFCC works best with category-level bidirectional calibration and fully disentangled advantage-space optimization.
\section{Conclusion and limitation}

\noindent\textbf{Conclusion.}
We reveal a thinking paradox in MER, where explicit reasoning does not necessarily improve recognition accuracy.
Our analysis shows that fast thinking provides recall-oriented coverage and correct-category confidence, while slow thinking offers precision-oriented selectivity.
Based on this complementarity, we propose \ourmethod, an RL framework that realizes slow-fast thinking synergy through dual-objective disentanglement and slow-fast confidence calibration.
Experiments on MER-UniBench and MME-Emotion show that \ourmethod~achieves SoTA performance and  fulfills the Desiderata, making the slow-thinking genuinely beneficial to emotion recognition.

\noindent\textbf{Limitation.}
This work is evaluated mainly on MER benchmarks, so its generalization to broader multimodal reasoning tasks remains open~\cite{Zhao_2025_ICCV,zhu2026look}.
Moreover, \ourmethod~depends on emotion-wheel-based category mapping and requires an additional fast-thinking forward pass during training, which may limit coverage of fine-grained emotions and increase computational cost.

\bibliographystyle{unsrtnat}
\bibliography{reference}

%%%%%%%%%%%%%%%%%%%%%%%%%%%%%%%%%%%%%%%%%%%%%%%%%%%%%%%%%%%%

\newpage
\appendix
{\LARGE\bf Appendix}
\appendix
\section{Theoretical analysis of dual-objective disentanglement}
\label{app:theory1}

We provide a theoretical analysis of the dual-objective disentanglement in
Section~\ref{sec:method1}. We first show that the standard $F_1$-based
advantage is affected by both the local sensitivities of $F_1$ and the reward variances. We then show that our two-step disentanglement removes these
factors by separating recall and precision before group normalization.

\subsection{Limitation of the $F_1$ advantage}

For the $i$-th response in a GRPO group, let $R_{\mathsf{R},i}$ and
$R_{\mathsf{P},i}$ denote its recall and precision rewards. The baseline first
combines them into a scalar $F_1$ reward:
\begin{equation}
R_{\mathsf{F},i}
=
\frac{2R_{\mathsf{R},i}R_{\mathsf{P},i}}
{R_{\mathsf{R},i}+R_{\mathsf{P},i}},
\end{equation}
and then computes the group-normalized advantage:
\begin{equation}
A_{\mathsf{F},i}
=
\frac{R_{\mathsf{F},i}-\mu_{\mathsf{F}}}{\sigma_{\mathsf{F}}},
\end{equation}
where $\mu_{\mathsf{F}}$ and $\sigma_{\mathsf{F}}$ are the group mean and
standard deviation of $R_{\mathsf{F}}$.

\begin{proposition}
For $\mathsf{k}\in\{\mathsf{R},\mathsf{P}\}$ and
$\mathsf{l}\neq\mathsf{k}$, the correlation between reward
$R_{\mathsf{k}}$ and the $F_1$ advantage $A_{\mathsf{F}}$ is approximated by:
\begin{equation}
\mathrm{Corr}(R_{\mathsf{k}},A_{\mathsf{F}})
\approx
\frac{
\alpha_{\mathsf{k}}\sigma_{\mathsf{k}}
+
\alpha_{\mathsf{l}}\rho\sigma_{\mathsf{l}}
}
{\sigma_{\mathsf{F}}}.
\end{equation}
Here, $\rho$ is the correlation between
$R_{\mathsf{R}}$ and $R_{\mathsf{P}}$;
$\alpha_{\mathsf{k}}$ is the local $F_1$ sensitivity to $R_{\mathsf{k}}$; and
$\sigma_{\mathsf{k}}$ is the standard deviation of $R_{\mathsf{k}}$.
\end{proposition}

\begin{proof}
For simplicity, we omit the response index $i$ when no ambiguity arises. Let
$\mu_{\mathsf{R}},\mu_{\mathsf{P}}$ be the group means of $R_{\mathsf{R}}$ and $R_{\mathsf{P}}$ respectively. By the first-order Taylor expansion around
$(\mu_{\mathsf{R}},\mu_{\mathsf{P}})$, the scalar $F_1$ reward can be locally
approximated as:
\begin{equation}
R_{\mathsf{F}}
\approx
\frac{2\mu_{\mathsf{R}}\mu_{\mathsf{P}}}
{\mu_{\mathsf{R}}+\mu_{\mathsf{P}}}
+
\alpha_{\mathsf{R}}(R_{\mathsf{R}}-\mu_{\mathsf{R}})
+
\alpha_{\mathsf{P}}(R_{\mathsf{P}}-\mu_{\mathsf{P}}),
\end{equation}
where:
\begin{equation}
\begin{aligned}
\alpha_{\mathsf{R}}
&=
\left.
\frac{\partial R_{\mathsf{F}}}{\partial R_{\mathsf{R}}}
\right|_{(\mu_{\mathsf{R}},\mu_{\mathsf{P}})}
=
\frac{2\mu_{\mathsf{P}}^2}
{(\mu_{\mathsf{R}}+\mu_{\mathsf{P}})^2},
\qquad
\alpha_{\mathsf{P}}
=
\left.
\frac{\partial R_{\mathsf{F}}}{\partial R_{\mathsf{P}}}
\right|_{(\mu_{\mathsf{R}},\mu_{\mathsf{P}})}
=
\frac{2\mu_{\mathsf{R}}^2}
{(\mu_{\mathsf{R}}+\mu_{\mathsf{P}})^2}.
\end{aligned}
\end{equation}
Subtracting the group mean from the local approximation and normalizing by
$\sigma_{\mathsf{F}}$ give:
\begin{equation}
R_{\mathsf{F}}-\mu_{\mathsf{F}}
\approx
\alpha_{\mathsf{R}}(R_{\mathsf{R}}-\mu_{\mathsf{R}})
+
\alpha_{\mathsf{P}}(R_{\mathsf{P}}-\mu_{\mathsf{P}}),
\qquad
A_{\mathsf{F}}
\approx
\frac{
\alpha_{\mathsf{R}}(R_{\mathsf{R}}-\mu_{\mathsf{R}})
+
\alpha_{\mathsf{P}}(R_{\mathsf{P}}-\mu_{\mathsf{P}})
}
{\sigma_{\mathsf{F}}}.
\end{equation}

We first compute the correlation between $R_{\mathsf{R}}$ and
$A_{\mathsf{F}}$. By definition:
\begin{equation}
\mathrm{Corr}(R_{\mathsf{R}},A_{\mathsf{F}})
=
\frac{
\mathrm{Cov}(R_{\mathsf{R}},A_{\mathsf{F}})
}
{
\sigma_{\mathsf{R}}\sigma_{A_{\mathsf{F}}}
}
=
\frac{
\mathrm{Cov}(R_{\mathsf{R}},A_{\mathsf{F}})
}
{\sigma_{\mathsf{R}}},
\end{equation}
where the equality follows from the group normalization of
$A_{\mathsf{F}}$, which gives $\sigma_{A_{\mathsf{F}}} = 1$. Using the expression of
$A_{\mathsf{F}}$, together with
$\mathrm{Var}(R_{\mathsf{R}})=\sigma_{\mathsf{R}}^2$ and
$\mathrm{Cov}(R_{\mathsf{R}},R_{\mathsf{P}})
=\rho\sigma_{\mathsf{R}}\sigma_{\mathsf{P}}$, we obtain:
\begin{equation}
\begin{aligned}
\mathrm{Cov}(R_{\mathsf{R}},A_{\mathsf{F}})
&\approx
\frac{
\alpha_{\mathsf{R}}\mathrm{Var}(R_{\mathsf{R}})
+
\alpha_{\mathsf{P}}\mathrm{Cov}(R_{\mathsf{R}},R_{\mathsf{P}})
}
{\sigma_{\mathsf{F}}}
=
\frac{
\alpha_{\mathsf{R}}\sigma_{\mathsf{R}}^2
+
\alpha_{\mathsf{P}}\rho
\sigma_{\mathsf{R}}\sigma_{\mathsf{P}}
}
{\sigma_{\mathsf{F}}}.
\end{aligned}
\end{equation}
By symmetry, the corresponding result for $R_{\mathsf{P}}$ can be obtained in the same way.
Substituting the two covariance results into the correlation definition gives:
\begin{equation}
\begin{alignedat}{2}
\mathrm{Corr}(R_{\mathsf{R}},A_{\mathsf{F}})
&\approx
\frac{
\alpha_{\mathsf{R}}\sigma_{\mathsf{R}}
+
\alpha_{\mathsf{P}}\rho\sigma_{\mathsf{P}}
}
{\sigma_{\mathsf{F}}},
\qquad&
\mathrm{Corr}(R_{\mathsf{P}},A_{\mathsf{F}})
&\approx
\frac{
\alpha_{\mathsf{P}}\sigma_{\mathsf{P}}
+
\alpha_{\mathsf{R}}\rho\sigma_{\mathsf{R}}
}
{\sigma_{\mathsf{F}}}.
\end{alignedat}
\end{equation}
This completes the proof.
\end{proof}

\noindent\textbf{Implication.} Proposition~2 shows that the correlation between each reward and the $F_1$
advantage depends on both the local $F_1$ sensitivity
$\alpha_{\mathsf{k}}$ and the reward standard deviation $\sigma_{\mathsf{k}}$,
especially when $\rho$ is small. Therefore, the baseline
tends to emphasize the reward with larger local sensitivity or larger standard
deviation, even when both recall and precision are important.

\subsection{Effect of dual-objective disentanglement}

\begin{proposition}
The correlation between each reward and the disentangled advantage is:
\begin{equation}
\mathrm{Corr}(R_{\mathsf{R}},A_{\mathsf{dual}})
=
\mathrm{Corr}(R_{\mathsf{P}},A_{\mathsf{dual}})
=
\sqrt{\frac{1+\rho}{2}}.
\end{equation}
\end{proposition}

\begin{proof}
By definition, $A_{\mathsf{R}}$ and $A_{\mathsf{P}}$ are standardized rewards:
\begin{equation}
A_{\mathsf{R}}
=
\frac{R_{\mathsf{R}}-\mu_{\mathsf{R}}}{\sigma_{\mathsf{R}}},
\qquad
A_{\mathsf{P}}
=
\frac{R_{\mathsf{P}}-\mu_{\mathsf{P}}}{\sigma_{\mathsf{P}}}.
\end{equation}
Thus, both have zero mean and unit variance, and their covariance is
$\mathrm{Cov}(A_{\mathsf{R}},A_{\mathsf{P}})
=
\mathrm{Corr}(R_{\mathsf{R}},R_{\mathsf{P}})
=
\rho$.
Since $A_{\mathsf{dual}}=A_{\mathsf{R}}+A_{\mathsf{P}}$, its variance is:
\begin{equation}
\mathrm{Var}(A_{\mathsf{dual}})
=
\mathrm{Var}(A_{\mathsf{R}})
+
\mathrm{Var}(A_{\mathsf{P}})
+
2\mathrm{Cov}(A_{\mathsf{R}},A_{\mathsf{P}})
=
2+2\rho.
\end{equation}

We first compute the correlation between $R_{\mathsf{R}}$ and
$A_{\mathsf{dual}}$. The covariance term can be written as:
\begin{equation}
\begin{aligned}
\mathrm{Cov}(R_{\mathsf{R}},A_{\mathsf{dual}})
&=
\mathrm{Cov}(R_{\mathsf{R}},A_{\mathsf{R}})
+
\mathrm{Cov}(R_{\mathsf{R}},A_{\mathsf{P}})
\\
&=
\mathrm{Cov}
\left(
R_{\mathsf{R}},
\frac{R_{\mathsf{R}}-\mu_{\mathsf{R}}}{\sigma_{\mathsf{R}}}
\right)
+
\mathrm{Cov}
\left(
R_{\mathsf{R}},
\frac{R_{\mathsf{P}}-\mu_{\mathsf{P}}}{\sigma_{\mathsf{P}}}
\right)
\\
&=
\sigma_{\mathsf{R}}
+
\rho\sigma_{\mathsf{R}}
=
(1+\rho)\sigma_{\mathsf{R}}.
\end{aligned}
\end{equation}
Combining this with
$\sigma_{A_{\mathsf{dual}}}=\sqrt{2+2\rho}$, we obtain:
\begin{equation}
\mathrm{Corr}(R_{\mathsf{R}},A_{\mathsf{dual}})
=
\frac{
\mathrm{Cov}(R_{\mathsf{R}},A_{\mathsf{dual}})
}
{
\sigma_{\mathsf{R}}\sigma_{A_{\mathsf{dual}}}
}
=
\frac{
(1+\rho)\sigma_{\mathsf{R}}
}
{
\sigma_{\mathsf{R}}\sqrt{2+2\rho}
}
=
\sqrt{\frac{1+\rho}{2}}.
\end{equation}
By symmetry, the same derivation gives:
\begin{equation}
\mathrm{Corr}(R_{\mathsf{P}},A_{\mathsf{dual}})
=
\sqrt{\frac{1+\rho}{2}}.
\end{equation}
This completes the proof.
\end{proof}

\noindent\textbf{Implication.} Proposition~3 shows that dual-objective
disentanglement removes the local-sensitivity and standard-deviation factors
identified in Proposition~2 by normalizing recall and precision separately. As
a result, both rewards are equally correlated with $A_{\mathsf{dual}}$, preventing
optimization from improving one objective at the expense of the other.

\subsection{Proof of Proposition~1}
\label{app:proof_prop1}

\begin{restatedproposition}[Proposition~\ref{prop:proposition1}]
Let $\mu_{\mathsf{R}},\mu_{\mathsf{P}}$ and
$\sigma_{\mathsf{R}},\sigma_{\mathsf{P}}$ denote the group means and standard
deviations of the recall and precision rewards, and let
$\tau=\frac{\sigma_\mathsf{R}/\mu_\mathsf{R}^2}{\sigma_\mathsf{P}/\mu_\mathsf{P}^2}$ be the normalized group-wise variance ratio.
The correlations between each reward and the standard $F_1$ advantage satisfy
\begin{equation}
\begin{cases}
\mathsf{corr}(R_{\mathsf{R}},A_{\mathsf{F}})
>
\mathsf{corr}(R_{\mathsf{P}},A_{\mathsf{F}}),
& \text{if } \tau>1, \\[1mm]
\mathsf{corr}(R_{\mathsf{R}},A_{\mathsf{F}})
\leq
\mathsf{corr}(R_{\mathsf{P}},A_{\mathsf{F}}),
& \text{if } \tau\leq 1.
\end{cases}
\label{eq:f1_corr_bias2}
\end{equation}
In contrast, our dual-objective advantage gives \textbf{balanced correlations}:
\begin{equation}
\mathsf{corr}(R_{\mathsf{R}},A_{\mathsf{dual}})
=
\mathsf{corr}(R_{\mathsf{P}},A_{\mathsf{dual}}).
\end{equation}
\end{restatedproposition}

\begin{proof}
From Proposition~2, the correlations between recall/precision rewards and the
standard $F_1$ advantage are:
\begin{equation}
\begin{alignedat}{2}
\mathrm{Corr}(R_{\mathsf{R}},A_{\mathsf{F}})
&\approx
\frac{
\alpha_{\mathsf{R}}\sigma_{\mathsf{R}}
+
\alpha_{\mathsf{P}}\rho\sigma_{\mathsf{P}}
}
{\sigma_{\mathsf{F}}},
\qquad&
\mathrm{Corr}(R_{\mathsf{P}},A_{\mathsf{F}})
&\approx
\frac{
\alpha_{\mathsf{P}}\sigma_{\mathsf{P}}
+
\alpha_{\mathsf{R}}\rho\sigma_{\mathsf{R}}
}
{\sigma_{\mathsf{F}}}.
\end{alignedat}
\end{equation}
Taking their difference gives:
\begin{equation}
\mathrm{Corr}(R_{\mathsf{R}},A_{\mathsf{F}})
-
\mathrm{Corr}(R_{\mathsf{P}},A_{\mathsf{F}})
\approx
\frac{
(1-\rho)
(\alpha_{\mathsf{R}}\sigma_{\mathsf{R}}
-
\alpha_{\mathsf{P}}\sigma_{\mathsf{P}})
}
{\sigma_{\mathsf{F}}}.
\end{equation}
In the non-degenerate case with $\rho<1$ and $\sigma_{\mathsf{F}}>0$, the sign of
this difference is determined by
$\alpha_{\mathsf{R}}\sigma_{\mathsf{R}}
-
\alpha_{\mathsf{P}}\sigma_{\mathsf{P}}$. Using the local sensitivities in
Proposition~2, we have:
\begin{equation}
\alpha_{\mathsf{R}}\sigma_{\mathsf{R}}
-
\alpha_{\mathsf{P}}\sigma_{\mathsf{P}}
=
\frac{2}
{(\mu_{\mathsf{R}}+\mu_{\mathsf{P}})^2}
\left(
\mu_{\mathsf{P}}^2\sigma_{\mathsf{R}}
-
\mu_{\mathsf{R}}^2\sigma_{\mathsf{P}}
\right).
\end{equation}
Therefore:
\begin{equation}
\alpha_{\mathsf{R}}\sigma_{\mathsf{R}}
>
\alpha_{\mathsf{P}}\sigma_{\mathsf{P}}
\Longleftrightarrow
\frac{\sigma_{\mathsf{R}}}{\mu_{\mathsf{R}}^2}
>
\frac{\sigma_{\mathsf{P}}}{\mu_{\mathsf{P}}^2}
\Longleftrightarrow
\tau>1.
\end{equation}
Thus, when $\tau>1$,
$\mathrm{Corr}(R_{\mathsf{R}},A_{\mathsf{F}})
>
\mathrm{Corr}(R_{\mathsf{P}},A_{\mathsf{F}})$; otherwise,
$\mathrm{Corr}(R_{\mathsf{R}},A_{\mathsf{F}})
\le
\mathrm{Corr}(R_{\mathsf{P}},A_{\mathsf{F}})$.

Finally, Proposition~3 directly gives:
\begin{equation}
\mathrm{Corr}(R_{\mathsf{R}},A_{\mathsf{dual}})
=
\mathrm{Corr}(R_{\mathsf{P}},A_{\mathsf{dual}})
=
\sqrt{\frac{1+\rho}{2}},
\end{equation}
which proves the balanced-correlation result for the dual-objective advantage.
\end{proof}

\section{Task details}
\label{appendix:task-details}

We evaluate \ourmethod~on two multimodal emotion benchmarks: MER-UniBench~\cite{lian2025affectgpt} and MME-Emotion~\cite{zhang2025mme}. 
MER-UniBench evaluates open-vocabulary multimodal emotion recognition under emotion-wheel-based metrics, while MME-Emotion provides a llm-as-a-judge based evaluation of both emotion recognition and emotion reasoning.

\subsection{MER-UniBench}
\label{appendix:merunibench}

MER-UniBench~\cite{lian2025affectgpt} unifies nine multimodal emotion datasets into three task groups: fine-grained emotion recognition, basic emotion recognition, and sentiment analysis.
Since model outputs are free-form emotion words, direct string matching is insufficient.
Following prior work~\cite{lian2025affectgpt,lian2024ov}, we evaluate recognition outputs with an emotion-wheel-based normalization protocol.

\noindent\textbf{Emotion-wheel normalization.}
The emotion wheel maps diverse fine-grained emotion words into a shared category space.
We apply three normalization steps to both predictions and ground-truth labels.
First, morphological variants are reduced to their base form, \eg, \emph{happier} and \emph{happiness} are normalized to \emph{happy}.
Second, semantically equivalent emotion words are mapped to the same canonical form, \eg, \emph{joyful} and \emph{happy}.
Third, fine-grained emotion words are further mapped to their corresponding inner emotion categories according to each emotion wheel.
Following prior work~\cite{lian2025affectgpt}, we adopt $K=5$ emotion wheels, as shown in Figure~\ref{fig:ew_appendix}.
In each wheel, the outer labels correspond to fine-grained emotions, and the inner sectors define the corresponding level-1 emotion categories.

\begin{figure}[h]
    \centering
    \setlength{\tabcolsep}{2pt}
    \begin{subcaptionbox}{W1}[0.19\linewidth]
        {\includegraphics[width=0.95\linewidth]{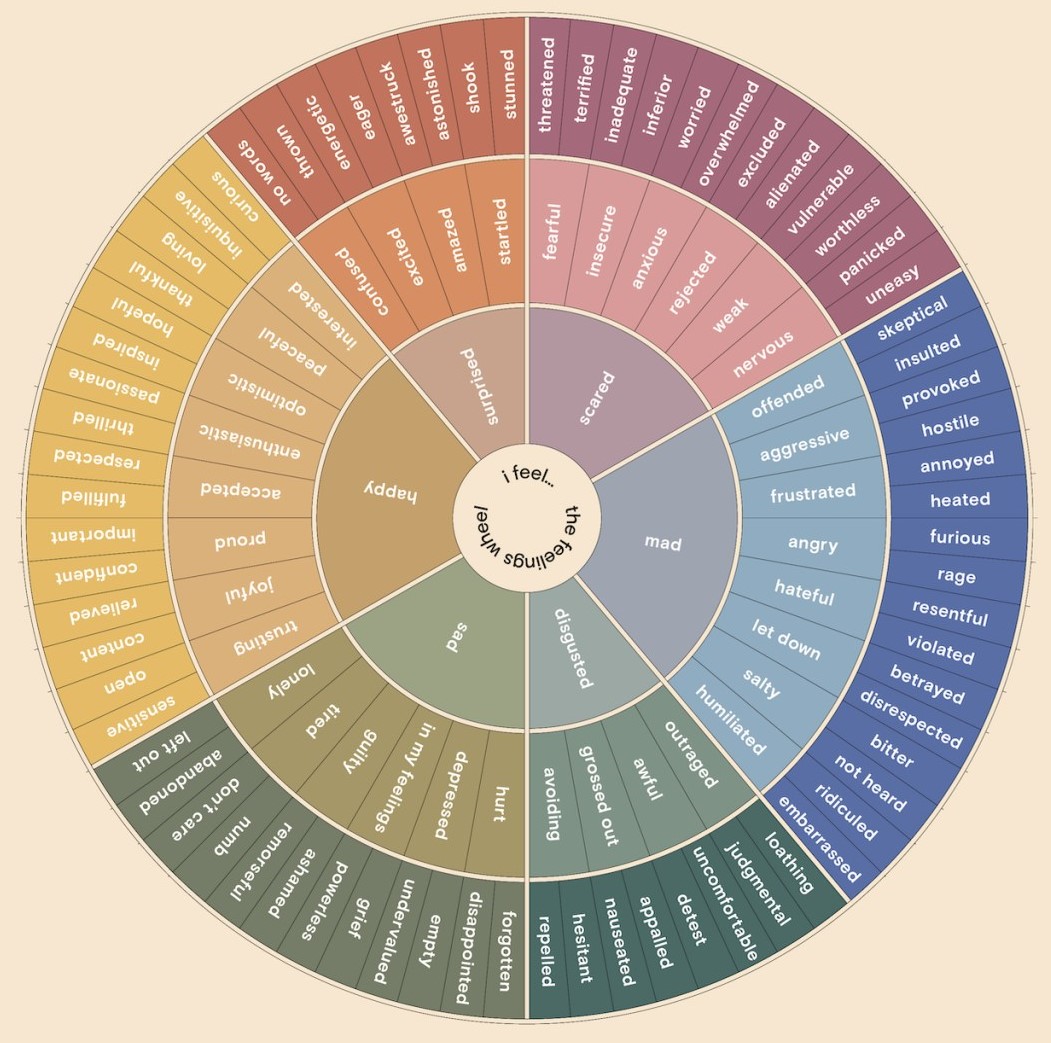}}
    \end{subcaptionbox}
    \begin{subcaptionbox}{W2}[0.19\linewidth]
        {\includegraphics[width=0.95\linewidth]{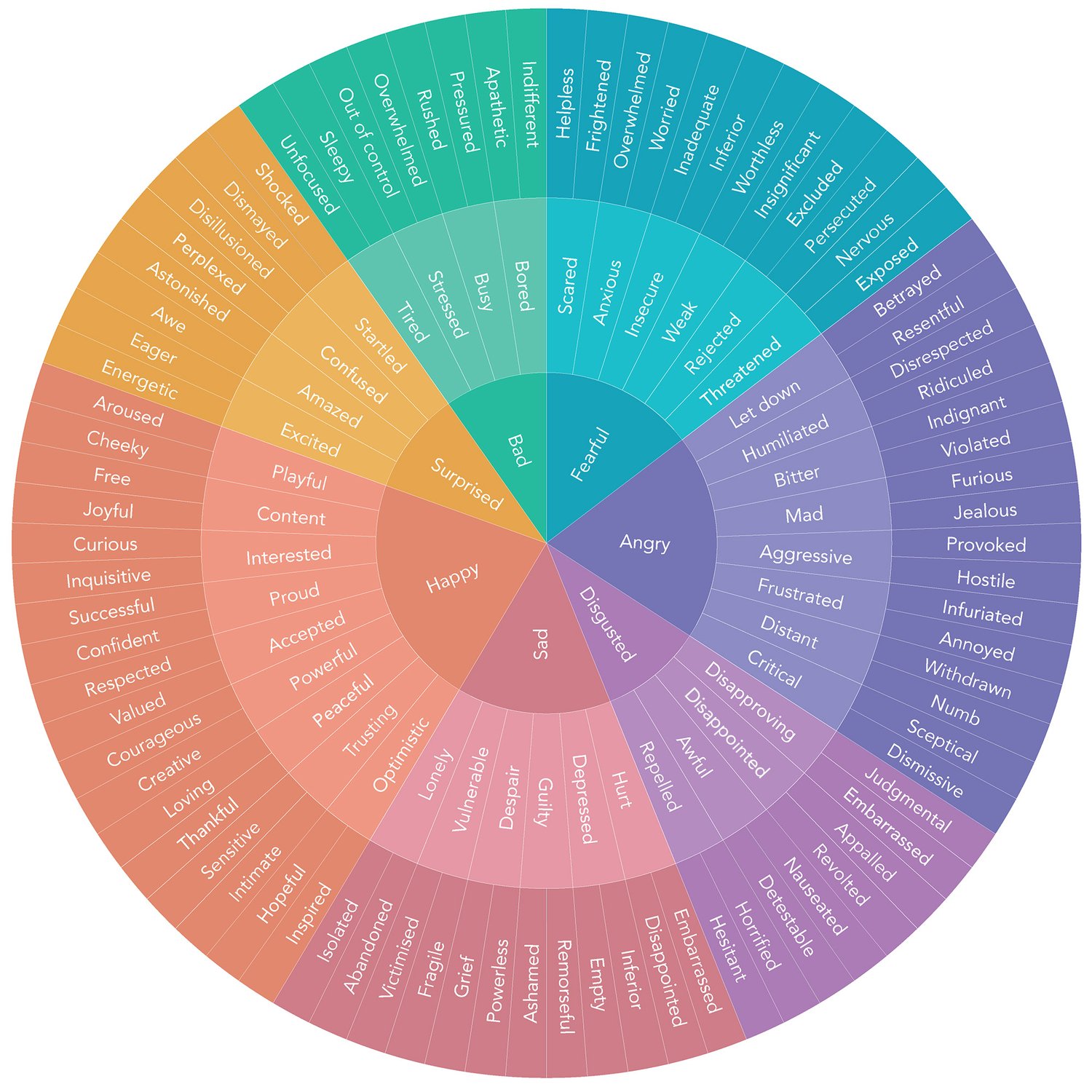}}
    \end{subcaptionbox}
    \begin{subcaptionbox}{W3}[0.19\linewidth]
        {\includegraphics[width=0.95\linewidth]{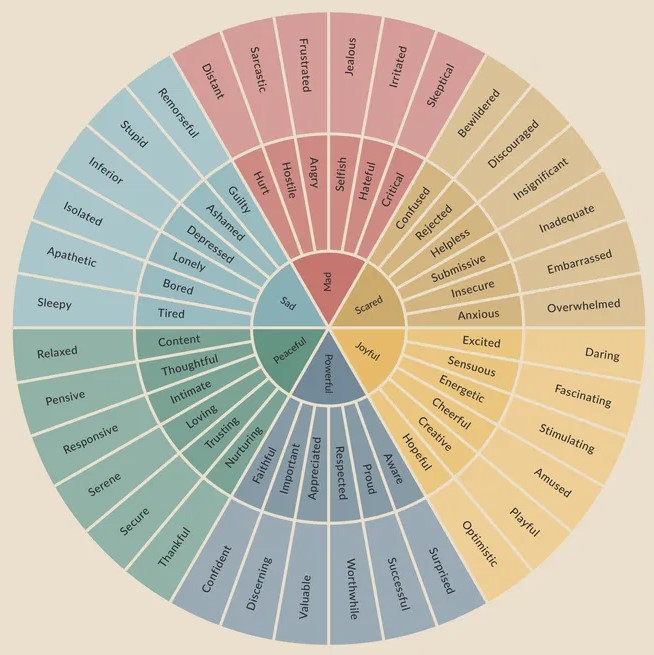}}
    \end{subcaptionbox}
    \begin{subcaptionbox}{W4}[0.19\linewidth]
        {\includegraphics[width=0.95\linewidth]{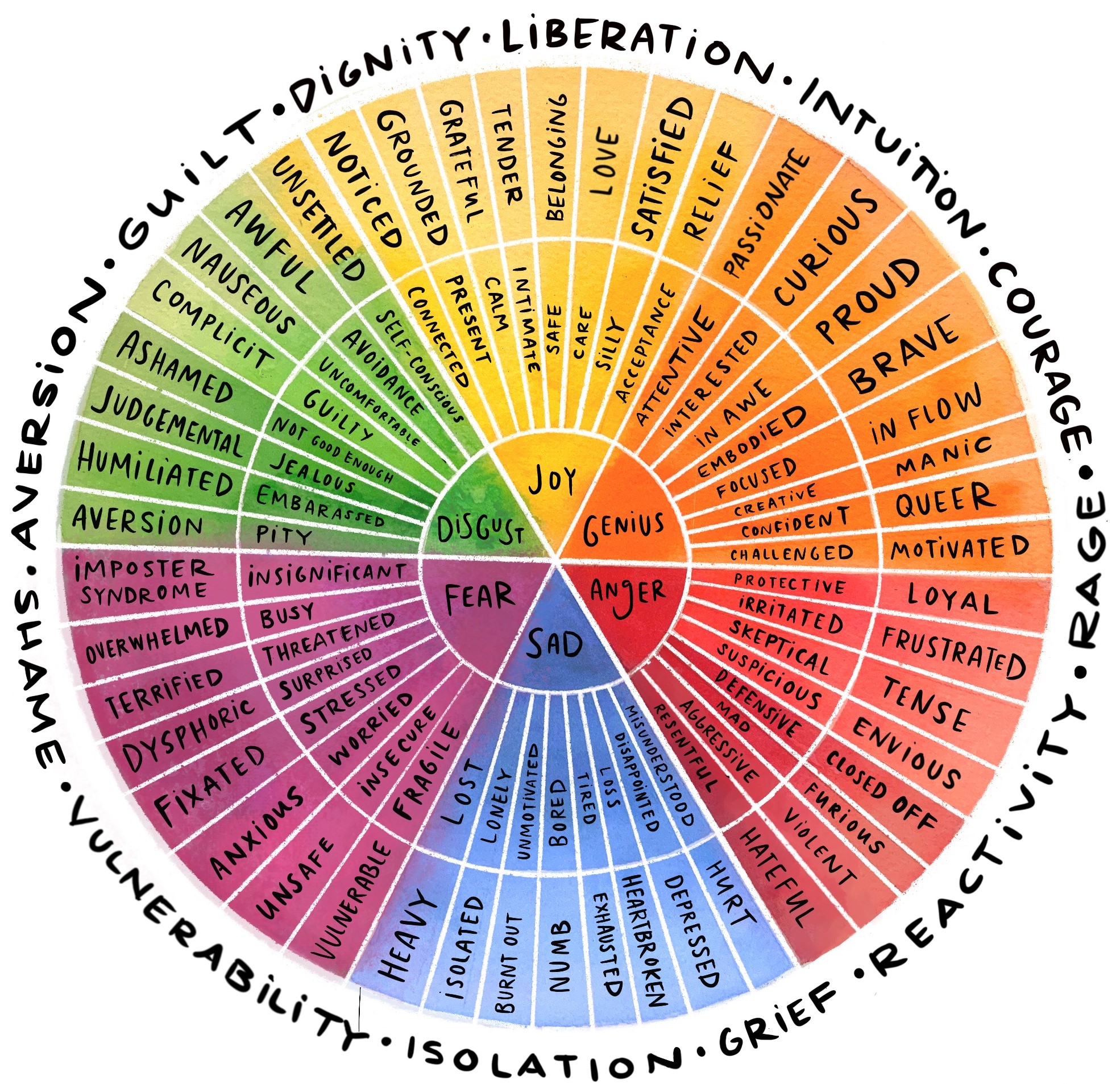}}
    \end{subcaptionbox}
    \begin{subcaptionbox}{W5}[0.19\linewidth]
        {\includegraphics[width=0.95\linewidth]{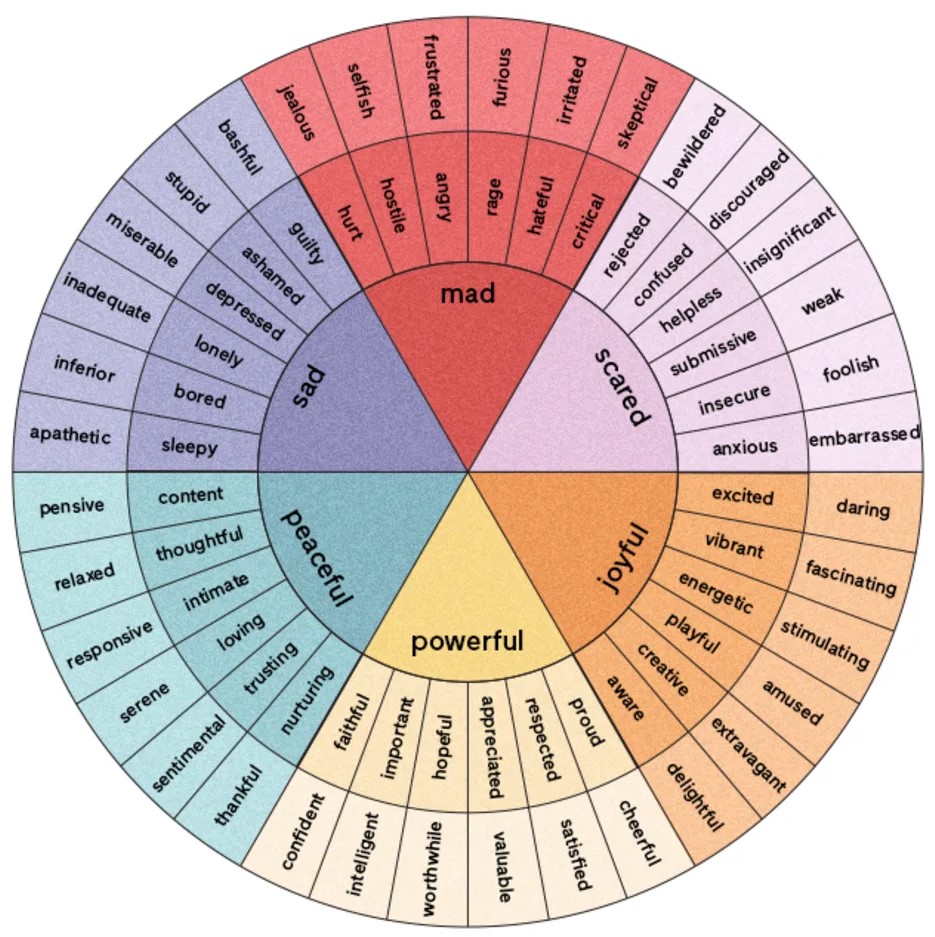}}
    \end{subcaptionbox}
    \caption{\textbf{Emotion wheels used in MER-UniBench evaluation.}
    We adopt five emotion wheels (W1--W5) to map free-form emotion words into level-1 emotion categories.}
    \label{fig:ew_appendix}
\end{figure}

For the $k$-th wheel $w_k$, the complete mapping function is defined as:
\begin{equation}
G_{w_k}(\cdot)
=
F_{l_3}^{w_k}
\bigl(
F_{l_2}(F_{l_1}(\cdot))
\bigr),
\quad k \in \{1,\ldots,K\},
\end{equation}
where $F_{l_1}$, $F_{l_2}$, and $F_{l_3}^{w_k}$ denote morphological normalization, semantic normalization, and wheel-specific category mapping, respectively.

\noindent\textbf{Set-level precision, recall, and F1.}
For sample $i$, let $\mathcal{Y}_i$ denote the ground-truth emotion set and 
$\hat{\mathcal{Y}}_i$ denote the predicted emotion set extracted from the 
\texttt{<answer>} field, with duplicate emotion words removed.
For each emotion wheel $w_k$, we compute the set-level precision and recall as:
\begin{equation}
\mathsf{P}_{s}^{k}
=
\frac{1}{N}\sum_{i=1}^{N}
\frac{
\left|
G_{w_k}(\mathcal{Y}_i)
\cap
G_{w_k}(\hat{\mathcal{Y}}_i)
\right|
}{
\left|
G_{w_k}(\hat{\mathcal{Y}}_i)
\right|
},
\end{equation}
\begin{equation}
\mathsf{R}_{s}^{k}
=
\frac{1}{N}\sum_{i=1}^{N}
\frac{
\left|
G_{w_k}(\mathcal{Y}_i)
\cap
G_{w_k}(\hat{\mathcal{Y}}_i)
\right|
}{
\left|
G_{w_k}(\mathcal{Y}_i)
\right|
}.
\end{equation}
The corresponding F1 score is defined as:
\begin{equation}
\mathsf{F}_{s}^{k}
=
\frac{
2\mathsf{P}_{s}^{k}\mathsf{R}_{s}^{k}
}{
\mathsf{P}_{s}^{k}+\mathsf{R}_{s}^{k}
}.
\end{equation}
The final emotion-wheel-based F1 score is obtained by averaging over all emotion wheels:
\begin{equation}
\mathsf{F}_{\mathsf{EW}}
=
\frac{1}{K}\sum_{k=1}^{K}\mathsf{F}_{s}^{k}.
\end{equation}

\noindent\textbf{Fine-grained emotion recognition.}
Fine-grained emotion recognition requires the model to predict one or more open-vocabulary emotion descriptors.
We evaluate this task on OV-MERD+~\cite{lian2024ov}, which extends OV-MERD~\cite{lian2023explainable}.
Following MER-UniBench, we use the emotion-wheel-based set-level F1 score as the primary metric.

\noindent\textbf{Basic emotion recognition.}
Basic emotion recognition asks the model to infer a basic emotion label from multimodal input.
We evaluate this task on MER2023~\cite{lian2023mer}, MER2024~\cite{lian2024mer}, MELD~\cite{poria2019meld}, and IEMOCAP~\cite{busso2008iemocap}.
Each sample is associated with a single majority-voted emotion label.
Under the official MER-UniBench protocol, this task is evaluated by Hitrate (HIT).
For each wheel $w_k$, HIT is defined as:
\begin{equation}
\mathrm{HIT}^{k}
=
\frac{1}{N}\sum_{i=1}^{N}
\mathbb{I}
\left[
G_{w_k}(y_i)
\in
G_{w_k}(\hat{\mathcal{Y}}_i)
\right],
\end{equation}
where $y_i$ is the ground-truth basic emotion label and $\mathbb{I}[\cdot]$ denotes the indicator function.
The final HIT score is averaged over the five emotion wheels.
Since HIT mainly reflects recall and does not penalize over-prediction, \textbf{we additionally report an F1-based evaluation for these four datasets} after the same emotion-wheel normalization.

\noindent\textbf{Sentiment analysis.}
Sentiment analysis requires the model to determine the polarity of the multimodal input.
We evaluate this task on CMU-MOSI~\cite{zadeh2016mosi}, CMU-MOSEI~\cite{zadeh2018multimodal}, CH-SIMS~\cite{yu2020ch}, and CH-SIMS v2~\cite{liu2022make}.
Since our model produces free-form emotion words rather than sentiment labels, we follow the MER-UniBench protocol and use Qwen2.5-7B-Instruct to map the extracted emotion words into a sentiment category, selected from \{\emph{positive}, \emph{negative}, \emph{neutral}\}.
We report the Weighted Average F-score (WAF) as the primary metric.
The prompt used for this mapping is shown below.

\begin{tcolorbox}[
    colback=blue!3,
    colframe=blue!33!black,
    colbacktitle=blue!9,
    coltitle=black,
    arc=2mm,
    auto outer arc,
    boxrule=0.35mm,
    left=1.5mm,
    right=1.5mm,
    top=1.5mm,
    bottom=1.5mm,
    width=\textwidth,
    fontupper=\small,
    title=\centering \textbf{Sentiment Mapping Prompt}
]
\textit{
Please assume the role of an expert in emotions. We provide a set of emotion words describing a character.
Please choose the most likely sentiment from the candidates: [positive, negative, neutral].
}
\end{tcolorbox}

\subsection{MME-Emotion}
\label{appendix:mme-emotion}

MME-Emotion~\cite{zhang2025mme} is a holistic benchmark for evaluating emotional intelligence in multimodal large language models.
Different from MER-UniBench, which mainly evaluates open-vocabulary emotion recognition with emotion-wheel-based metrics, MME-Emotion uses LLM-based judging to assess both the final emotion prediction and the reasoning process.

\noindent\textbf{Tasks.}
MME-Emotion contains eight evaluation tasks:
emotion recognition in laboratory scenarios (ER-Lab), emotion recognition in the wild (ER-Wild), emotion recognition under noisy conditions (Noise-ER), fine-grained emotion recognition (FG-ER), multi-label emotion recognition (ML-ER), sentiment analysis (SA), fine-grained sentiment analysis (FG-SA), and intent recognition (IR).
These tasks cover both controlled and in-the-wild scenarios, and evaluate different aspects of affective understanding, including categorical emotion recognition, fine-grained affective interpretation, sentiment polarity, and intent inference.

\noindent\textbf{Evaluation metrics.}
MME-Emotion reports three metrics.
\emph{Recognition} measures whether the final predicted emotion or affective label is correct.
\emph{Reasoning} evaluates whether the generated explanation is coherent, grounded, and supportive of the final answer.
\emph{CoT} is the average of Recognition and Reasoning.
Following the benchmark protocol, task-level results are judged by an evaluator model.
Since the original GPT-4o evaluator used by the benchmark is no longer available in our setting, we use \texttt{gemini-3.1-flash-lite-preview} as the evaluator for all compared methods to ensure a consistent evaluation protocol.
In the Table~\ref{tab:mme_emotion_main}, we report the CoT score for each task and the averaged Recognition, Reasoning, and CoT scores across all tasks.

\section{Implementation details}
\label{app:training_details}

We use Qwen2.5-Omni~\cite{xu2025qwen2} as the backbone and train only its Thinker module.
All models are trained on MER-Caption+~\cite{lian2025affectgpt} with a two-stage SFT-RL pipeline based on GRPO-style optimization.
We consider two reasoning baselines.
The reproduced AffectGPT-R1~\cite{lian2025affectgptr1} follows its original SFT/RL data allocation, while keeping the backbone and optimization framework consistent with our implementation.
Our stronger \emph{Baseline} adopts an RL-heavy allocation: we randomly sample 5k MER-Caption+ instances for SFT and use the remaining samples for RL.
For this \emph{Baseline} and \ourmethod, we train the SFT stage for 2 epochs with a learning rate of $2\times 10^{-5}$, and train the RL stage for 1 epoch with a learning rate of $2\times 10^{-6}$.
During RL, we sample $G=4$ responses per prompt, train for 3262 optimization steps, set the batch size to 1, and use gradient accumulation with 2 steps.
We set both $\lambda_{\mathsf{cali}}$ and $\lambda_{\mathsf{fmt}}$ to $0.1$.
All experiments are conducted on 16 NVIDIA H100 GPUs, and the RL stage takes approximately 18 hours.
For fair comparison, we evaluate all datasets using a single final checkpoint without dataset-specific checkpoint selection.

\begin{figure}[t]
    \centering
    \begin{minipage}{0.5\linewidth}
        \centering
        \includegraphics[width=\linewidth]{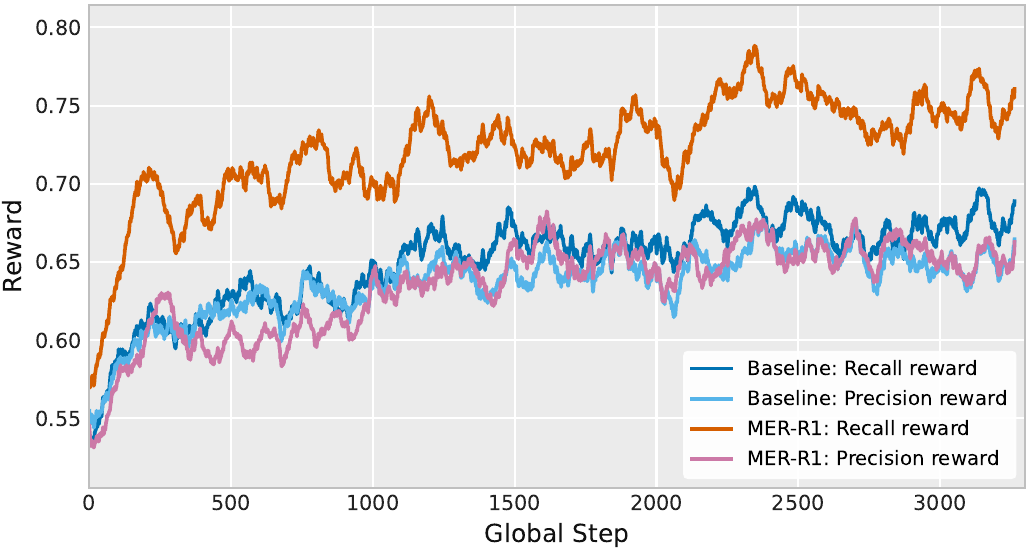}
        \caption{Training dynamics on dual rewards.}
        \label{fig:training_dynamics_rp}
    \end{minipage}
    \hfill
    \begin{minipage}{0.49\linewidth}
        \centering
        \includegraphics[width=\linewidth]{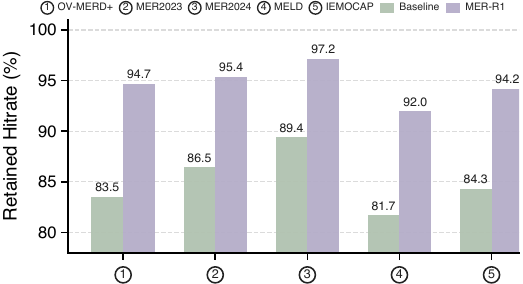}
        \caption{Retained hitrate comparision.}
        \label{fig:retained_hit}
    \end{minipage}
    \vspace{-1em}
\end{figure}

\section{Additional analysis}
\label{app:additional_analysis}

\noindent\textbf{Training dynamics on dual rewards.}
Figure~\ref{fig:training_dynamics_rp} compares the recall and precision reward dynamics of the \textit{Baseline} and \ourmethod.
\ourmethod~achieves a substantially higher recall reward, reaching around $0.75$ compared with about $0.67$ for the \textit{Baseline}.
Meanwhile, both methods maintain similar precision rewards, fluctuating around $0.65$.
This shows that \ourmethod~improves recall without sacrificing precision.
The dual-objective design therefore better balances the recall--precision trade-off and preserves the recall-oriented coverage that is often suppressed during slow thinking.

\noindent\textbf{Retained hitrate analysis.}
We define retained hitrate as the fraction of ground-truth emotion categories that are correctly predicted by fast thinking and still preserved in the final slow-thinking answer.
As shown in Figure~\ref{fig:retained_hit}, \ourmethod~consistently retains more correct fast-thinking predictions than the \textit{Baseline} across all five recognition benchmarks.
For example, the retained hitrate increases from $83.5$ to $94.7$ on OV-MERD+, from $86.5$ to $95.4$ on MER2023, and from $89.4$ to $97.2$ on MER2024, with similar gains on MELD and IEMOCAP.
These results show that \ourmethod~better preserves useful fast-thinking intuition during slow thinking, instead of filtering out correct categories as the \textit{Baseline} often does.
This supports the role of slow-fast confidence calibration in selectively retaining recall-oriented signals from fast thinking.

\section{Qualitative analysis}
\label{app:qualitative_analysis}

We provide qualitative examples to further illustrate how \ourmethod~realizes slow-fast thinking synergy in multimodal emotion recognition.
Specifically, we present three types of cases.
(i) \emph{Fast-intuition retention} refers to cases where fast thinking already predicts a correct emotion, but the baseline slow-thinking answer removes it after deliberation; \ourmethod~keeps this useful fast-thinking intuition in the final answer.
(ii) \emph{Spurious emotion suppression} refers to cases where fast thinking introduces incorrect emotions due to over-coverage; \ourmethod~suppresses these spurious predictions through slow-thinking selectivity.
(iii) \emph{Missed emotion recovery} refers to cases where baseline fast thinking fails to cover the ground-truth emotion, while \ourmethod~recovers it after slow reasoning.

For each case type, we show examples from both fine-grained emotion recognition and basic emotion recognition.
The fast-intuition retention examples show that \ourmethod~does not simply discard the useful predictions produced by fast thinking; instead, it preserves correct categories such as \textbf{angry} or \textbf{happy} that are missed by the baseline slow-thinking answer.
The spurious emotion suppression examples show that \ourmethod~can remove unsupported emotions such as \textbf{angry} or \textbf{resentment}, keeping the final prediction focused on the ground-truth emotions.
The missed emotion recovery examples further show that \ourmethod~can recover emotions overlooked by baseline fast thinking, such as \textbf{surprised}, \textbf{worried}, and other positive fine-grained emotions.
Together, these examples qualitatively demonstrate that \ourmethod~selectively combines the recall-oriented intuition of fast thinking with the precision-oriented selectivity of slow thinking.

\begin{table*}[ht]
\centering
\caption{Fine-grained emotion recognition example of fast-intuition retention.}
\label{tab:example1_1}

\begin{minipage}{1.0\textwidth}
\centering
\resizebox{\linewidth}{!}{
\begin{tabular}{>{\centering\arraybackslash}m{2.4cm}|>{\arraybackslash}m{14cm}}
\toprule
\multicolumn{2}{l}{\bf Example of fast-intuition retention} \\
\midrule

&
\raisebox{-0.9\height}{\includegraphics[height=3cm]{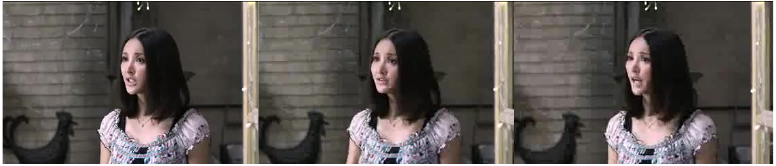}} \\
\midrule

Subtitle (En) & {I tell you, it's not easy for a woman who has divorced and has a child to find a partner, it's not as easy as you think.} \\
\midrule

Ground Truth & Labels: \textbf{angry}, dissatisfied \\
\midrule

\makecell[c]{Baseline\\Fast Thinking} &
\texttt{<think>} \texttt{</think>}

\texttt{<answer>}\textbf{angry}, frustrated, anxious, disappointed\texttt{</answer>} \\
\midrule

\makecell[c]{Baseline\\Slow Thinking} &
\texttt{<think>}In the text, the subtitle content says, "I'll tell you, it's not easy for a woman who has divorced and is raising a child to find a partner." Based on the woman's relaxed and comfortable posture and facial expression in the video clues, as well as the individual's high level of emotional stress and frustration in the audio clues, we can infer that this sentence may carry a negative or depressed tone. This sentence may be the woman's evaluation or complaint about the current situation or some difficulties. The woman's positive emotions and comfortable posture in the video clues indicate that she may be able to cope with the current situation or solve problems in some way, but based on the individual's negative emotions in the audio clues, we can infer that this sentence may carry a sense of disappointment or frustration. Therefore, this sentence may be the woman's negative evaluation of the current situation, aligning with her overall display of negative emotions.\texttt{</think>}

\texttt{<answer>}disappointment, frustration\texttt{</answer>} \\
\midrule

\makecell[c]{\ourmethod\\Slow Thinking} &
\texttt{<think>}In the text, the subtitle reads: "It's difficult for a woman to find a partner after divorce." This sentence may be the speaker's response to a certain problem or situation. Based on the audio clues describing the speaker's aggressive tone and high-pitched voice, as well as the audio clue of the speaker expressing feelings of anger or frustration, we can infer that this sentence may carry a tone of complaint or dissatisfaction. The speaker may be angry or dissatisfied about the difficulties or troubles of a woman finding a partner after divorce, hence using this statement to express their emotions.\texttt{</think>}

\texttt{<answer>}\textbf{angry}, frustrated, anxious, stressed\texttt{</answer>}

\smallskip
\noindent\textit{Explanation:} Baseline slow thinking misses the ground-truth emotion \textbf{angry}, while \ourmethod~slow thinking retains it.
\\
\bottomrule

\end{tabular}
}
\end{minipage}
% \vspace{-10mm}
\end{table*}
\begin{table*}[ht]
\centering
\caption{Basic emotion recognition example of fast-intuition retention.}
\label{tab:example1_1}

\begin{minipage}{1.0\textwidth}
\centering
\resizebox{\linewidth}{!}{
\begin{tabular}{>{\centering\arraybackslash}m{2.4cm}|>{\arraybackslash}m{14cm}}
\toprule
\multicolumn{2}{l}{\bf Example of fast-intuition retention} \\
\midrule

&
\raisebox{-0.9\height}{\includegraphics[height=3cm]{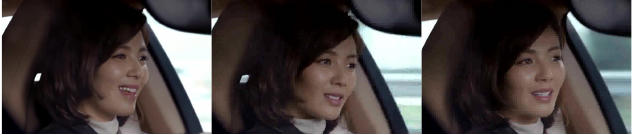}} \\
\midrule

Subtitle (En) & {Thank you all, thank you for not treating me as an outsider. Actually, before this, I thought I only had one friend in my life, Old Tan. I didn't expect to meet you all on the 22nd floor.} \\
\midrule

Ground Truth & Labels: \textbf{happy} \\
\midrule

\makecell[c]{Baseline\\Fast Thinking} &
\texttt{<think>} \texttt{</think>}

\texttt{<answer>}\textbf{happy}, surprised, grateful\texttt{</answer>} \\
\midrule

\makecell[c]{Baseline\\Slow Thinking} &
\texttt{<think>}In the text, the subtitle reads: "Thank you all, thank you for not treating me as an outsider. Actually, before this, I thought I only had one friend in my life, Old Tan, but I didn't expect to meet you all on the 22nd floor." This sentence may be the woman expressing gratitude to the people around her. Based on the audio clue of the woman's voice being filled with excitement and enthusiasm, as well as the video clue of the woman smiling and laughing while driving, we can infer that this sentence may carry a tone of excitement and gratitude. The woman may be surprised and grateful for the people around her, feeling that she has made more friends than she thought. Therefore, this sentence aligns with the overall positive emotions displayed by the woman.\texttt{</think>}

\texttt{<answer>}excitement, gratitude, surprise\texttt{</answer>}\\
\midrule

\makecell[c]{\ourmethod\\Slow Thinking} &
\texttt{<think>}In the text, the subtitle reads: "Thank you all, thank you for not treating me as an outsider. Actually, before this, I thought I only had one friend, Mr. Tan. But I didn't expect to meet you all on the 22nd floor." This sentence might be a woman's expression of gratitude and excitement to those around her. Based on the audio clues, with the woman's voice having a high pitch and slightly faster pace, as well as the video clues showing the woman's happiness and excitement, we can infer that this sentence might be delivered with a positive, appreciative, and joyful tone. The woman might be sharing her happiness and gratitude with those around her, while also feeling unexpected and fortunate for meeting new friends.\texttt{</think>}

\texttt{<answer>}\textbf{happy}, grateful, appreciative, positive\texttt{</answer>}

\smallskip
\noindent\textit{Explanation:} Baseline slow thinking misses the ground-truth emotion \textbf{happy}, while \ourmethod~slow thinking retains it.
\\
\bottomrule

\end{tabular}
}
\end{minipage}
\end{table*}
\begin{table*}[ht]
\centering
\caption{Fine-grained emotion recognition example of spurious emotion suppression.}
\label{tab:example1_1}

\begin{minipage}{1.0\textwidth}
\centering
\resizebox{\linewidth}{!}{
\begin{tabular}{>{\centering\arraybackslash}m{2.4cm}|>{\arraybackslash}m{14cm}}
\toprule
\multicolumn{2}{l}{\bf Example of spurious emotion suppression} \\
\midrule

&
\raisebox{-0.9\height}{\includegraphics[height=3cm]{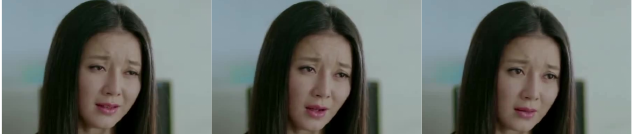}} \\
\midrule

Subtitle (En) & {You would feel better this way, wouldn't you?} \\
\midrule

Ground Truth & Labels: concerned, negative, sad \\
\midrule

\makecell[c]{Baseline\\Fast Thinking} &
\texttt{<think>} \texttt{</think>}

\texttt{<answer>}\textbf{angry}, frustrated, anxious, sad, disappointed\texttt{</answer>} \\
\midrule

\makecell[c]{\ourmethod\\Slow Thinking} &
\texttt{<think>}In the text, the subtitle reads, "Can you feel better this way?" This sentence may be a woman's response or concern for someone. Based on the audio clue describing the woman's crying and heavy breathing, as well as the video clue of the woman's sad and crying facial expression, we can infer that this sentence may carry a tone of concern or comfort. The woman may be responding to someone's distress or negative emotions, hoping that the other person can feel better or cope with the situation. Therefore, this sentence may express the woman's concern and care for the other person.\texttt{</think>}

\texttt{<answer>}concerned, worried, sad, anxious\texttt{</answer>}

\smallskip
\noindent\textit{Explanation:} Baseline fast thinking predicts the spurious emotion \textbf{angry}, whereas \ourmethod~slow thinking suppresses it and keeps the prediction focused on ground-truth emotions.
\\
\bottomrule

\end{tabular}
}
\end{minipage}
\end{table*}
\begin{table*}[ht]
\centering
\caption{Basic emotion recognition example of spurious emotion suppression.}
\label{tab:example1_1}

\begin{minipage}{1.0\textwidth}
\centering
\resizebox{\linewidth}{!}{
\begin{tabular}{>{\centering\arraybackslash}m{2.4cm}|>{\arraybackslash}m{14cm}}
\toprule
\multicolumn{2}{l}{\bf Example of spurious emotion suppression} \\
\midrule

&
\raisebox{-0.9\height}{\includegraphics[height=3cm]{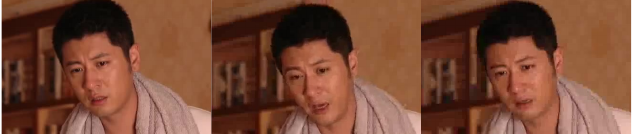}} \\
\midrule

Subtitle (En) & {Grandpa.} \\
\midrule

Ground Truth & Labels: sad \\
\midrule

\makecell[c]{Baseline\\Fast Thinking} &
\texttt{<think>} \texttt{</think>}

\texttt{<answer>}anxiety, nervousness, worry, confusion, frustration, disappointment, stress, sadness, \textbf{anger}, \textbf{resentment}, urgency, despair, panic\texttt{</answer>} \\
\midrule

\makecell[c]{\ourmethod\\Slow Thinking} &
\texttt{<think>}In the text, the caption reads: "Grandpa." This phrase might be a man's way of addressing someone or a title. Based on the video hints, with the man's facial expression showing contemplation and determination, and his body language suggesting relaxation, we can infer that this statement could be the man expressing concern, worry, or asking for help or assistance to his grandfather. The audio clues describing his voice with a slight tremble, at a slightly faster pace, and at a higher pitch, suggest he might be experiencing intense emotions such as anxiety, excitement, or stress. Therefore, this sentence could be the man expressing his emotional distress or seeking help.\texttt{</think>}

\texttt{<answer>}anxiety, sadness, concern, worry\texttt{</answer>}

\smallskip
\noindent\textit{Explanation:} Baseline fast thinking predicts the spurious emotion \textbf{anger}, \textbf{resentment}, whereas \ourmethod~slow thinking suppresses it and keeps the prediction focused on ground-truth emotions.
\\
\bottomrule

\end{tabular}
}
\end{minipage}
% \vspace{-8mm}
\end{table*}
\begin{table*}[ht]
\centering
\caption{Fine-grained emotion recognition example of missed emotion recovery.}
\label{tab:example1_1}

\begin{minipage}{1.0\textwidth}
\centering
\resizebox{\linewidth}{!}{
\begin{tabular}{>{\centering\arraybackslash}m{2.4cm}|>{\arraybackslash}m{14cm}}
\toprule
\multicolumn{2}{l}{\bf Example of missed emotion recovery} \\
\midrule

&
\raisebox{-0.9\height}{\includegraphics[height=3cm]{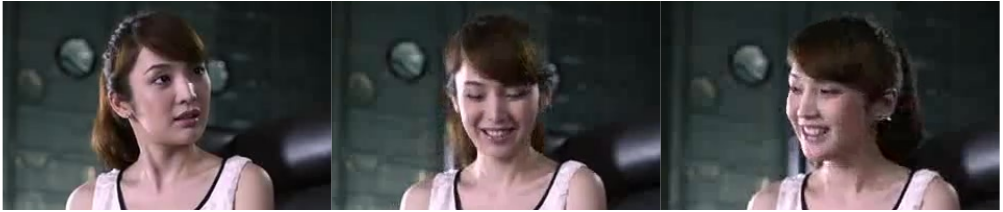}} \\
\midrule

Subtitle (En) & {It's hard to see it.} \\
\midrule

Ground Truth & Labels: relaxed, shy, happy, teasing \\
\midrule

\makecell[c]{Baseline\\Fast Thinking} &
\texttt{<think>} \texttt{</think>}

\texttt{<answer>}excited, anxious, frustrated, stressed, stressed out\texttt{</answer>} \\
\midrule

\makecell[c]{\ourmethod\\Slow Thinking} &
\texttt{<think>}In the text, caption content: "It's hard to see it." This sentence might be a woman's response or expression towards the man. Based on the audio clues, which include a woman's laughter and her voice filled with excitement and anticipation, it can be inferred that the woman is experiencing positive emotions and is comfortable around the man. The video hints at the woman's comfortable and relaxed state. From the woman's voice in the audio clues being high-pitched and enthusiastic, it can be speculated that she is experiencing positive emotions and might be responding to the man's words or actions in a playful or teasing manner. Therefore, this sentence could be a woman's sarcastic or humorous way of responding, fitting well with her overall positive emotions.\texttt{</think>}

\texttt{<answer>}happy, content, playful, teasing, amused\texttt{</answer>}

\smallskip
\noindent\textit{Explanation:} Baseline fast thinking misses all ground-truth emotions, whereas \ourmethod~slow thinking recovers the correct positive emotions, including \textbf{happy}, \textbf{content}, and \textbf{teasing}.
\\
\bottomrule

\end{tabular}
}
\end{minipage}
% \vspace{-8mm}
\end{table*}
\begin{table*}[ht]
\centering
\caption{Fine-grained emotion recognition example of missed emotion recovery.}
\label{tab:example1_1}

\begin{minipage}{1.0\textwidth}
\centering
\resizebox{\linewidth}{!}{
\begin{tabular}{>{\centering\arraybackslash}m{2.4cm}|>{\arraybackslash}m{14cm}}
\toprule
\multicolumn{2}{l}{\bf Example of missed emotion recovery} \\
\midrule

&
\raisebox{-0.9\height}{\includegraphics[height=3cm]{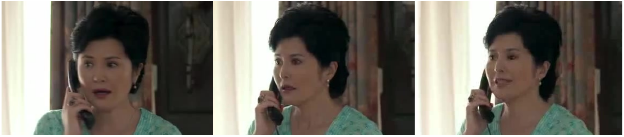}} \\
\midrule

Subtitle (En) & {Ah! We won't see each other on Sunday?} \\
\midrule

Ground Truth & Labels: surprised \\
\midrule

\makecell[c]{Baseline\\Fast Thinking} &
\texttt{<think>} \texttt{</think>}

\texttt{<answer>}disappointment, frustration, anxiety\texttt{</answer>} \\
\midrule

\makecell[c]{\ourmethod\\Slow Thinking} &
\texttt{<think>}In the text, the subtitle reads: "Ah, we're not going to see each other on Sunday." Based on the audio clues, which describe the speaker's voice with a high pitch and excited tone, along with the speaker's use of "ah" to express emotional distress, we can infer that this sentence might convey a sense of anxiety or surprise. The speaker's high pitch and excited tone suggest that they are experiencing strong emotions, while the "ah" in the subtitle could be the expression of surprise or confusion. Therefore, this sentence might reflect the speaker's concern or disappointment over the lack of a Sunday meeting.\texttt{</think>}

\texttt{<answer>}confused, anxious, surprised, frustrated\texttt{</answer>}

\smallskip
\noindent\textit{Explanation:} Baseline fast thinking misses the ground-truth emotion \textbf{surprised}, while \ourmethod~slow thinking successfully recovers this missed emotion in its final prediction.
\\
\bottomrule

\end{tabular}
}
\end{minipage}
% \vspace{-8mm}
\end{table*}
\begin{table*}[!t]
\centering
\caption{Basic emotion recognition example of missed emotion recovery.}
\label{tab:example1_1}
\vspace{-0.5em}

\begin{minipage}{1.0\textwidth}
\centering
\resizebox{\linewidth}{!}{
\begin{tabular}{>{\centering\arraybackslash}m{2.4cm}|>{\arraybackslash}m{14cm}}
\toprule
\multicolumn{2}{l}{\bf Example of missed emotion recovery} \\
\midrule

&
\raisebox{-0.9\height}{\includegraphics[height=3cm]{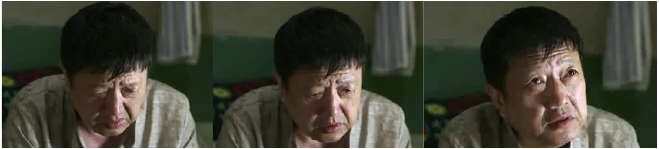}} \\
\midrule

Subtitle (En) & {Are you afraid that my presence here will affect your interaction with Cao Xiaoqiang?} \\
\midrule

Ground Truth & Labels: worried \\
\midrule

\makecell[c]{Baseline\\Fast Thinking} &
\texttt{<think>} \texttt{</think>}

\texttt{<answer>}angry, anxious, frustrated\texttt{</answer>} \\
\midrule

\makecell[c]{\ourmethod\\Slow Thinking} &
\texttt{<think>}{In the text, the caption reads: "Are you really that afraid of me here." This sentence might be a man's response to someone's concern or unease. Based on the audio cues describing a personal voice with a flat and monotonous tone, a higher pitch and slightly more aggressive nature, as well as the video clues showing the man's serious and thoughtful expression, we can infer that the statement might convey a tone of worry or concern. The man might feel uneasy or stressed because of the other person's question or response, thus this sentence could be his way of responding to express his own worry or anxiety.\texttt{</think>}}

\texttt{<answer>}concerned, fearful, worried, anxious\texttt{</answer>}

\smallskip
\noindent\textit{Explanation:} Baseline fast thinking misses the ground-truth emotion \textbf{worried}, whereas \ourmethod~slow thinking recovers precise emotions such as \textbf{worried} and \textbf{concerned}, with other predictions remaining contextually plausible.
\\
\bottomrule

\end{tabular}
}
\end{minipage}
\end{table*}

%%%%%%%%%%%%%%%%%%%%%%%%%%%%%%%%%%%%%%%%%%%%%%%%%%%%%%%%%%%%

\end{document}